\documentclass{article}

\usepackage[utf8]{inputenc}
\usepackage[T1]{fontenc}
\usepackage{microtype}
\usepackage{csquotes}
\usepackage[osf,sc]{mathpazo}
\RequirePackage[scaled=0.90]{helvet}
\RequirePackage[scaled=0.85]{beramono}
\RequirePackage{textcomp}

\usepackage[letterpaper,top=2cm,bottom=2cm,left=2.7cm,right=2.7cm,marginparwidth=1.75cm]{geometry}

\usepackage[skip=2mm,indent=5mm]{parskip}

\usepackage{natbib}

\usepackage{xcolor}
\definecolor{DarkGreen}{rgb}{0.1,0.5,0.1}
\definecolor{DarkRed}{rgb}{0.5,0.1,0.1}
\definecolor{DarkBlue}{rgb}{0.1,0.1,0.5}
\definecolor{Gray}{rgb}{0.2,0.2,0.2}

\usepackage[small]{caption}
\usepackage[pdftex]{hyperref}
\hypersetup{
    unicode=false,          
    pdftoolbar=true,        
    pdfmenubar=true,        
    pdffitwindow=false,      
    pdfnewwindow=true,      
    colorlinks=true,       
    linkcolor=DarkBlue,          
    citecolor=DarkGreen,        
    filecolor=DarkRed,      
    urlcolor=DarkBlue,          
    pdftitle={},
    pdfauthor={},
}

\usepackage{jcd}
\usepackage{gary_macros}
\usepackage{algorithm}
\usepackage{algpseudocode}
\usepackage{enumitem}
\usepackage{subfigure}

\usepackage{booktabs}
\usepackage{multirow}
\usepackage{makecell}
\usepackage{xspace}

\usepackage{tikz}
\usetikzlibrary{arrows.meta}
\usepackage{MnSymbol}
\crefname{assumption}{\protect{assumption}}{\protect{assumption}}
\Crefname{assumption}{\protect{Assumption}}{\protect{Assumption}}

\title{Causal Inference out of Control:\\ Estimating the Steerability of Consumption}
\author{Gary Cheng\footnote{Stanford University} , \;  Moritz Hardt\footnote{Max Planck Institute for Intelligent Systems, T\"ubingen, and T\"ubingen AI Center} , \; and Celestine Mendler-D\"unner\footnotemark[2]}
\date{}

\begin{document}

\maketitle

\begin{abstract}
Regulators and academics are increasingly interested in the causal effect that algorithmic actions of a digital platform have on consumption. We introduce a general causal inference problem we call the \emph{steerability of consumption} that abstracts many settings of interest. Focusing on observational designs and exploiting the structure of the problem, we exhibit a set of assumptions for causal identifiability that significantly weaken the often unrealistic overlap assumptions of standard designs. The key novelty of our approach is to explicitly model the dynamics of consumption over time, viewing the platform as a controller acting on a dynamical system. From this dynamical systems perspective, we are able to show that exogenous variation in consumption and appropriately responsive algorithmic control actions are sufficient for identifying steerability of consumption. Our results illustrate the fruitful interplay of control theory and causal inference, which we illustrate with examples from econometrics, macroeconomics, and machine learning.
\end{abstract}

\newcommand{\remove}[1]{}

\section{Introduction}\label{sec:motivation}

How much do advertisements decrease screen time? Do algorithmic recommendations increase consumption of inflammatory content? Does exposure to diverse news sources mitigate political polarization? 
These are a few questions that firms, researchers, and regulators alike ask about digital platforms \citep{Barber2015TweetingFL,Brown2022EchoCR}. We unify these questions under the task we term: estimating the \emph{steerability of consumption}---i.e., estimating the effect of platform actions on consumer behavior.

Estimating the steerability of consumption requires causal inference because past consumption and platform actions influence both future consumption and future actions.
In other words, they introduce confounding.
Resolving confounding through randomization in the form of A/B tests is standard in the industry. However, randomization is not always possible on digital platforms. As past experience shows, experiments may be ethically fraught~\citep{KramerGuHa14, editorialcomment14}, technically challenging to implement, or prohibitively expensive.  Moreover, external investigators may simply not have the power to experimentally intervene in the practices of a platform.
Observational causal inference is a promising alternative. 
However, standard observational causal designs
do require the observed data satisfy an overlap assumption: the data generating distribution must assign positive probability to treatment in all strata defined by any realizable choice of the confounders. 
But since the interaction of participants with digital platforms often spans multiple time steps, the confounding set could become very large. High dimensional confounders make overlap unlikely to hold~\citep{DAmourDiFeLeSe17}, ultimately resulting in invalid inferences. An additional challenge is that algorithmic platform actions are not randomized treatments: the actions they take are strongly correlated with---or in some cases---deterministic functions of the data observed, making overlap assumptions with respect to past consumer and platform actions even less likely to hold.

To address these challenges, we take advantage of the structure of the interaction between digital platforms and their participants to expose weaker assumptions that permit \emph{valid} observational causal inference.
To do so, we take a control-theoretic perspective on the problem of estimating the steerability of consumption. Rather than omitting the role of time, as is common in causal inference, we explicitly keep track of the interactions between the platform and the participants over time.
In particular, we model consumption as a dynamical system where the consumer's features $x_t$ at time~$t$ are determined by the platform action $u_{t-1}$, the previous state $x_{t-1}$, as well as exogenous noise. The platform's action $u_t$ is then updated based on the most recent observations of $x_t$. 
As a concrete example, let $x_t$ measure what a consumer clicks on and $u_t$ as what a recommender system suggests. Applied to this example, our model captures the time-dependent interplay between user and recommender system.
Our model posits that the dynamics are Markovian---that the current time step is only affected by the previous time step---which serves to reduce the dimension of the confounder. We argue this assertion is reasonable for digital settings, as future recommendations are dictated largely by consumption in the recent past.
Building on this model, we demonstrate that it is possible to circumvent directly assuming exogenous variations in platform actions in order to establish overlap and identifiability of the steerability of consumption. We show that a) sufficient exogenous variation on the consumer's features and b) the platform control action being non-degenerate, is sufficient for identifiability.
We emphasize that, in contrast to standard approaches, our results hold even when the platform's action is a deterministic function of the past consumption and actions (e.g., a predictive model), a plausible setting in digital systems.

\begin{figure*}[t]
\small
    \centering
    \subfigure[standard model]{
    \begin{tikzpicture}[x=0.7cm,y=0.6cm]
    \begin{scope}[every node/.style={circle,thick,draw,minimum size=9mm, fill = black!10!white}]
        \node (A) at (0,0) {$u$};
        \node (D) at (3,3) {$x$};
    \end{scope}
    \begin{scope}[every node/.style={circle,thick,draw,minimum size=9mm}]
        \node (C) at (0,3) {$z$};
    \end{scope}
    \begin{scope}[>={Stealth[black]},
                  every node/.style={fill=white,circle},
                  every edge/.style={draw=black,very thick}]
        \path [->] (C) edge (D);
        \path [->] (C) edge (A);
    \end{scope}
    \begin{scope}[>={Stealth[orange]},
                  every node/.style={fill=white,circle},
                  every edge/.style={draw=orange,very thick}]
        \path [->] (A) edge (D);
    \end{scope}
    \end{tikzpicture}\label{fig:general-causal-graph}}
    \hspace{2cm}
    \subfigure[modeling temporal confounding structure]{
    \begin{tikzpicture}[x=0.7cm,y=0.6cm]
    \begin{scope}[every node/.style={circle,thick,draw,minimum size=8mm}]
        \node (B) at (6,0) {$u_{t-2}\hid $};
        \node (D) at (6,3) {$x_{t-2}\hid $};
        \node (E) at (9,3) {$x_{t-1}\hid $};
    \end{scope}
    \begin{scope}[every node/.style={thick,minimum size=9mm}]
        \node (G) at (6,5) {$\noise_{t-2}$} ;
        \node (H) at (9,5) {$\noise_{t-1}$} ;
        \node (M) at (12,5) {$\noise_{t}$} ;
    \end{scope}
    \begin{scope}[every node/.style={circle,thick,draw,minimum size=8mm, fill = black!10!white}]
        \node (L) at (9,0) {$u_{t-1}\hid $};
        \node (K) at (12,3) {$x_{t}\hid $};
    \end{scope}
    \node at (3,0) {$\dots$};
    \node at (3,3) {$\dots$};    
    \begin{scope}[>={Stealth[black]},
                  every node/.style={fill=white,circle},
                  every edge/.style={draw=black,very thick}]
        \path [->] (4.5,3) edge (D);
        \path [->] (D) edge (E);
        \path [->] (E) edge (K);
        \path [->] (E) edge (K);
        \path [->] (E) edge (L);
        \path [->] (4.5,0) edge (B);
    \end{scope}
    \begin{scope}[>={Stealth[black]},
                  every node/.style={fill=white,circle},
                  every edge/.style={draw=black,very thick}]
        \path [->] (G) edge (D);
        \path [->] (H) edge (E);
        \path [->] (M) edge (K);
        \path [->] (D) edge (B);
        \path [->] (B) edge (L);
        \path [->] (B) edge (E);
        \path [->] (4.5,1.5) edge (D);
    \end{scope}
    \begin{scope}[>={Stealth[orange]},
                  every node/.style={fill=white,circle},
                  every edge/.style={draw=orange,very thick}]
        \path [->] (L) edge (K);
    \end{scope}
        \end{tikzpicture}
        \label{fig:autoregressive-causal-graph}}
    \caption{The causal inference problem of estimating the steerability of consumption.}
\end{figure*}
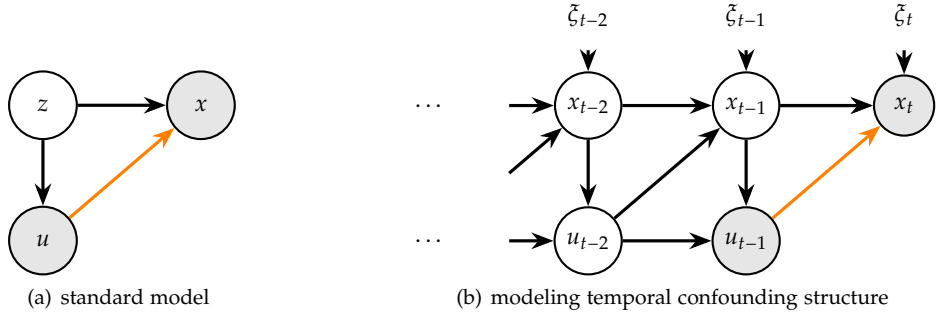

\paragraph{Contributions.} We unify a class of important causal inference problems under the umbrella of steerability of consumption. We propose a time-aware dynamical systems model to study these problems, and we design associated assumptions for observational causal inference. Working with our model,
we establish necessary and sufficient conditions for identifiability of the steerability of consumption. We demonstrate that sufficient exogenous variation in consumption and sufficient expressivity in the platform response enable causal identification, circumventing the need for direct interventions or exogenous variation on the platform action. 
We show that exogenous variation in consumption at two time steps is sufficient for identifiability, whereas one consumption shock, in general, is not.
We analyze two estimators---the \twostageregression and the adjustment formula estimators---for estimating the steerability of consumption from finite samples.
Finally, we experiment on real data to test the efficacy of our Markovian assumption at reducing overlap violations.

Practitioners routinely apply causal inference methods well outside the guardrails of typical assumptions. 
Our work can be seen as a route towards justifying the valid use of observational causal inference for estimating  steerability of consumption. Along the way, we connect problems of causal inference with the technical repertoire of control theory, a fruitful avenue for further research.

\subsection{Background}

The fact that digital platforms, their predictions, and their actions non-trivially impact the individuals that interact with the platform has widely been recognized in diverse applications spanning content recommendation, prediction policy problems and labor markets~\citep[c.f.,][]{shmueli20, thai16traffic, fleder10recom,admoavicius13, Krauth2022BreakingFL, Barber2015TweetingFL,Brown2022EchoCR}. 
In the machine learning community, the implications of predictions on populations have formally been studied in several works~\citep[e.g.,][]{PerdomoZrMeHa20, Dean2022PreferenceDU, Kalimeris2021PreferenceAI, Chaney2018HowAC}. 
We point out the work by \citet{HardtJaMe22}. They relate the extent to which a platform can steer user behavior to the economic concept of power, and introduce performative power to quantify it. Assessing performative power crucially relies on estimating the causal effect of algorithmic actions. Thus, our work provides sufficient conditions for how performative power can be assessed from observational data. 
Related to our work, \citet{mendler22causal} also focus on identifying the causal effect of predictions on eventual outcomes in settings where the covariates and the prediction are deterministically bound. However, they do not take advantage of repeated interactions between the predictor and the population, but instead take advantage of potential incongruences in modality.
Similarly, estimating the steerability of consumption has also been the motivation of a recent work on causal inference in the presence of confounding by~\cite{shah22steer}. However, the authors focus on dealing with partially unobserved confounding $z$, while taking overlap in the rollout for granted by assuming that the joint distribution $p(u,x,z)$ belongs to an exponential family. 

Our modeling approach is inspired by the literature on dynamical systems in control theory. Taking this perspective, the task of estimating the steerability of consumption in our causal model maps to a system identification problem~\citep{Ljung10}. However, our problem setup differs from the standard control theory setting because we focus on purely observational designs, where we do not choose what platform control actions (i.e., interventions) are taken. Within the system identification literature, we highlight the work of \cite{Abbasi-YadkoriSz11} because of the similarity of their model to the linear model we study in \Cref{sec:linear-model}. Their work proposes a method of controlling linear quadratic control systems with unknown dynamics via the principle of certainty equivalence; their results hinge on a finite-sample system identification result, similar in spirit to the type of identifiability results found in this paper.

From a technical standpoint the causal question we are interested in is related to studies of dose response and treatment-effect estimation under overlap violations in causal inference~\citep[c.f.,][]{PetersenPoGrWaVa12}. By approaching the problem from a control theoretic angle we arrive at a principled approach to shrink the adjustment set and make identifiability possible.

\section{Model}
\label{sec:general-model}

The standard causal model for our problem is shown in \Cref{fig:general-causal-graph}. Estimating the steerability of consumption corresponds to quantifying the causal effect of a platform action $u$ on a state $x$, subject to observed confounding $z$, where actions $u$ represent the algorithmic decisions of a digital platform, and the variable $x$ captures relevant user features, such as what content the user consumed. The confounding variable $z$ captures all available past information that influences both the choice of platform action $u$ and the variable $x$. 
As we have explained earlier, high dimensional confounding due to long rollouts and correlated platform actions suggest overlap is unlikely to hold in the standard setting, making the standard model unsuitable for estimating the steerability of consumption. 

The unique feature of our model---outlined in \Cref{fig:autoregressive-causal-graph}---is that it makes the temporal component of interactions among the confounding variables explicit. 
We let $x_t\hid \in \R^d$ and $u_t\hid  \in \R^p$ denote the consumption and platform action at time step $t$ respectively. 
We assume for all $t\geq 0$ the dynamics of the system follow 
\begin{align}
\begin{split}
    x_t\hid &= \Hx{f(x_{t-1}\hid)}{ g(u_{t-1}\hid)} + \noise_t\\
  u_t\hid &= \Hu{h(x_t\hid)}{ r(u_{t-1}\hid)}
\end{split}
\label{eqn:gen_model}   
\end{align}
with $\noise_t \in \R^d$ modeling potential exogenous variations in $x_t$ and the functions $f: \R^d\to \R^d$, $g: \R^p \to \R^d$, $h: \R^d  \to \R^p$, and $r: \R^p \to \R^p$ describe how consumption and platform actions affect one another.  
%
 We make the following assumption on the exogenous noise:\footnote{We choose to use \Cref{ass:independent_noise} for clarity, even though it is stronger than we need for our results. See \Cref{sec:independence-discussion} for a discussion of how to relax the assumption.}
\begin{assumption}[Mutually Independent Exogenous Variation]\label{ass:independent_noise}
    For any $t \geq 1$, the random variable $\noise_t$ is mutually independent of $\noise_{k}$ for all $k\neq t$ and independent of $(x_{0}, u_{0})\sim P_{0}$. 
\end{assumption}
With respect to the model we outlined above, we define steerability of consumption as the ability of the platform to change user consumption. More formally, given a time step $t$, a base action $u$, and an intervention $u'$, we define the steerability of consumption as
\begin{align*}
    \steer_t(u,u')\defeq\E[x_{t} \mid \mathrm{do}(u_{t-1} \defeq u')] - \E[x_{t} \mid \mathrm{do}(u_{t-1} \defeq u)].
\end{align*}
In our model, a sufficient condition for identifying the steerability of consumption is to identify the following causal effect
\[\bar{x}_t(u)\defeq \E[x_{t} \mid \mathrm{do}(u_{t-1} \defeq u)]\,.\]
Because our system dynamics~\eqref{eqn:gen_model} are time-invariant and the structural equations for $x$ are assumed to be separable, we have $\steer_t=\steer_{t'}$ for all $t, t'$. 
Thus, without loss of generality, we will focus on identifying $\steer_\bT$ via identifying $\bar{x}_\bT$, letting $\bT$ denote the index we are interested in estimating the steerability of consumption.
For $\bK\geq 1$, we use $R_\bK$ to denote a rollout of the previous $\bK$ time indices leading up to the chosen time index $\bT$: 
\[R_\bK\defeq(\{x_{\bT-t}, u_{\bT - t}\}_{t=1}^{\bK}, x_\bT).\]
In this work, we assume access to iid observations of rollouts $R_\bK$. We will specify $\bK$ in each result.

\subsection{Running example}
\label{sec:running-example}
We instantiate our model with an example. Consider an auditor who is interested in estimating the impact of the recommendation algorithm of a video streaming platform---like Twitch or YouTube---on the consumption patterns of its users. Let $y_t \in \R^p$ be some measure of content consumption (e.g., number of hours streamed) for $p$ video categories of interest during week $t$ for a given user. 
Let $z_t \in \R^{d_z}$ be comprised of measurements about the platform such as revenue per category, click-through rate per category, unique weekly users, unique advertisers per category, competitors' performance, etc. which could be confounders. We can think of the joint vector $[y_t; z_t] \in \R^d$ as the state variable $x_t$ for $d=p+d_z$. The platform action $u_t \in \R^p$ is a measure of how many videos from the $p$ categories of interest are recommended to a given user
during week $t$. The platform interfaces using $u_t$ with the goal of maximizing total profits, which is some deterministic function of $x_t$. The auditor is interested in estimating how the platform action $u_{t-1}$ impacts the average watch habits $y_t$ of users. More specifically, they are interested in the first $p$ coordinates of the steerability of consumption $\steer(u, u')$. 
    
Our model postulates that user consumption changes over time based on the recommendations by the algorithm, as well as external factors (e.g., new trends).  Formally, taking inspiration from \cite{JamborWaLa12}, we model the dynamics of the system as 
\begin{align*}
    z_t & = f_{1}(z_{t-1}, y_{t-1}) + g_1(u_{t-1}) + \noise_t^{(1)}\\
    y_t &= f_{2}(z_{t-1}, y_{t-1}) + g_2(u_{t-1})+ \noise_t^{(2)}.
\end{align*}
The function~$f_1$ models how the performance metrics chosen as a target variable by the firm evolve over time, while the function~$g_1$ models the platform's ability to control this metric. The function~$f_{2}$ models how much interest users retain in each video category from week to week, as well as the effect of confounders on viewership (e.g., how many hours of viewing time can a competitor poach). 
The auditor wants to estimate the relationship~$g_2$ that governs how much consumption increases as more recommendations get served. The noise variables~$\noise_t^{(1)}, \noise_t^{(2)}$ allow for natural variation in user preferences. For example, the price of Bitcoin may increase due to changes in economic conditions, leading to many more users watching cryptocurrency videos; this change in behavior is independent of past consumption and the platform's recommendations. 
We can model the platform action similarly as
\[u_t = h(z_t, y_t) + r(u_{t-1}),\]
where $h$ models the platform's algorithm of how viewer statistics and other metrics affect  recommendations in the future. 
The function $r$ models how the video streaming service regularizes its recommendations to avoid overfitting to recent activity. 

\paragraph{Plausibility of modeling assumptions.} 
Our model posits a Markovian assumption on the platform and consumption dynamics and an assumption that the consumption and platform action updates are additive (separable) in nature. 
The Markovian assumption on the platform action dynamics is reasonable for two main reasons.
First, digital platforms are constantly retraining machine learning models on fresh data as a way to improve performance, mitigate distribution shift, and quickly fix models which have suffered unexpected drops in performance \citep{Shankar2022OperationalizingML}. Given that this retraining occurs on a daily or even hourly cadence, this suggests that the machine-learning-based algorithmic control actions a platform takes at any time $t$ 
mainly depend on the state and actions from the recent past.
Second, the Markovian view of digital platform control actions is an accepted view in the recommendation system literature. For example, the contextual multi-armed bandit models used to study recommendation systems are Markovian by construction---the platform uses fresh context provided at every time step to make its decisions \citep{Langford2007TheEA, Bouneffouf2019ASO}.
The Markovian assumption on consumer dynamics is based on the belief that there are few long range causal effects that affect consumption, and that the ones that do exist---say inherent biases, interests, or habits---can be encoded directly or by proxy into all of the states, without blowing up the dimension. For example, we could encode long-term, content-specific click habits by estimating click proportions by content category and placing this information into all of the states.
We can generalize our non-linear results to settings beyond additive-update dynamics. We choose to focus on additive updates because a) it is the simplest model which still conveys the nuance of our results, b) it is well accepted in the dynamical system and causal inference literature, and c) additive updates are prevalent in machine learning (e.g., gradient methods).

\paragraph{Beyond recommender systems.} 
The steerability of consumption is not a term specific to recommender systems; rather, it is a general term referring to the impact platform actions have on user behavior. It certainly is applicable to other digital settings. For example, many digital advertising platforms (and third-party auditors) are interested in whether personalized advertising increases platform activity. On one hand, advertisements clutter user interfaces, making the user experience less streamlined, but on the other hand, personalized advertisements provide users with more opportunities to engage, giving the platform more influence over user lives. To model this scenario, let $x_t$ be some measure of engagement (e.g., clicks, time online) and $u_t$ be some measure of the type and quantity of ads served. Confounders could include other platform performance measures such as monthly active users. 
Besides digital platforms, our model also applies to some economic settings. Micro-economists are often interested in estimating the effect product prices have on demand, termed the \textit{price elasticity of demand}. If we model product demand using $x_t$ and model product prices using $u_t$, then the price elasticity of demand is precisely the steerability of consumption. Confounders like product quality can be accounted for in the state variable $x_t$. In macroeconomics, a classical problem is estimating the effect the Federal Interest Rate has on inflation and unemployment. We can use our framework to model the Federal Interest Rate as the platform action $u_t$ and inflation and unemployment rates as the state $x_t$. In this example, GDP and other measures of the global economy could be possible confounders to account for.

\section{Identifiability from exogenous variations on consumption}
\label{sec:identifiability}

In this section, we outline necessary and sufficient conditions for identifying $\steer_\bT$ given iid observations of $R_{\bK=1}$.
A quantity is \textit{identifiable} if it can be uniquely determined from observational data probability distribution. Conversely, if there exists multiple values of said quantity which are all consistent with the observational data probability distribution, then we say it is \textit{unidentifiable}. 

To provide some context and intuition for our proof strategy for showing identifiability, let us start from the general causal graph in Figure~\ref{fig:general-causal-graph} and recall classical results from causal inference in the presence of observed confounding~\citep{pearl09book}. 
Standard results tell us that a sufficient condition for identifiability of the causal effect of $u$ on $x$  is \emph{admissibility} and \emph{overlap}. 

\begin{definition}[Admissibility]\label{def:adjustment-formula}
We say a continuous random variable $Z$ with density $p$ is admissible for adjustment with respect to treatment $U$ and outcome $X$ if the adjustment formula is valid: 
\begin{equation}\label{eqn:adjustment-formula}
    \E[X \mid \mathrm{do}(U \defeq u)] = \int \E[X \mid U=u, Z=z] p(z) \,dz.
\end{equation}
\end{definition}

\begin{definition}[Overlap]
Given an action $U$ and a confounding variable $Z$ with well-defined joint density $p$. Then, we say overlap of $(u,z)$ is satisfied if $p_{U \mid Z}(u' \mid z') >0$ for all $u' \in \R^p$ and $z'$ where $p_Z(z') > 0$.
\end{definition}

Overlap guarantees that every $z$ in the support has non-zero probability to co-occur any action $u$, and thus $\E[X\mid U=u, Z=z]$ is well defined.
Overlap with admissibility guarantees that $\E[X \mid \mathrm{do}(U \defeq u)]$ can be uniquely expressed as a function of observational data distributions, via \eqref{eqn:adjustment-formula}, implying $\E[X \mid \mathrm{do}(U \defeq u)]$ is identifiable.

Now, we return to our model. In order to show $\bar{x}_\bT$ is identifiable, we first show that $x_{\bT-1}$ is admissible for adjustment. 
The proof of \Cref{lem:admissibility} is found in \Cref{proof:lem:admissibility}.

\begin{proposition}[Admissibility in our model]\label{lem:admissibility}
Given the structural equations in \eqref{eqn:gen_model} and let \Cref{ass:independent_noise} hold.  Then, $x_{\bT-1}$ is admissible with respect to $u_{\bT}$ and $x_{\bT}$ for any $\bT \geq 0$.
\end{proposition}

Hence, the main challenge for establishing identifiability of $\steer_\bT$ is to argue about overlap of $(u_{\bT-1},x_{\bT-1})$. Once we show overlap, we can rewrite $\bar{x}_\bT$ as a function of the observational probability distribution (of $R_{\bK =1}$) by way of the adjustment formula \eqref{eqn:adjustment-formula}.
This would mean $\bar{x}_\bT$ is identifiable and therefore the steerability of consumption $\steer_\bT$ is as well.

\subsection{Key assumptions}
\label{sec:key-assm}

We highlight the two requirements on the dynamical system in \eqref{eqn:gen_model} that will allow us to establish overlap of $(u_{\bT-1},x_{\bT-1})$. The first assumption requires that there is exogenous noise in the system that leads to sufficient variation in consumption $x$ across time. 

\begin{definition}[Consumption shock]\label{ass:noise-coverage}
For a given time step $t\geq 0$ we say there is a consumption shock at time $t$, if the noise $\noise_t$ satisfies $p_{\noise_t}(a) > 0$ for all $a \in \R^d$ where $p_{\noise_t}$ denotes the density of $\noise_t$. 
\end{definition}

 We say the system is exposed to $\bM$ shocks prior to $\bT$ if for all $t \in \{\bT-\bM, \ldots, \bT-1\}$, there is a shock in consumption. We expect that variations in consumption naturally occur in the presence of unexpected news events, economic shocks, or new trends.
In order to leverage these consumption shocks for the purpose of identifiability, we need one crucial assumption on the platform action, which will allow us to circumvent directly assuming exogenous variation on the platform action.
Namely, the platform needs to be sufficiently sensitive to the variations in consumption $x$, so that the consumption shocks propagate into the platform action $u$ at consecutive time steps.

\begin{definition}[Responsive platform action]\label{ass:expressive-control}
For a platform, let $q_c: \R^d \to \R^p$ defined as 
$q_c(y) \defeq r(h(y) + c)$ 
describe how the current state $y$ affects the next platform action, given that the previous platform action was $c$. If $q_c$ is a surjective, continuously differentiable map with a Jacobian $J \in \R^{p, d}$ such that $\rank(J) = \min(p, d)$ always holds for all $c\in \R^p$, then we say the platform action is responsive.
\end{definition}

To put our assumptions in context, recall the video recommender system example from Section~\ref{sec:running-example}.  
We expect that variations in user video consumption (\Cref{ass:noise-coverage}) naturally occur in the presence of unexpected news events, economic shocks, or new trends. 
To investigate \Cref{ass:expressive-control}, consider $r(u) = \alpha u$ as a plausible example. This corresponds to a model where the platform uses previous platform actions as a regularizer for how they select future actions. Note that this simple choice of $r$ is surjective.  Furthermore, we expect the number of metrics and confounders which can be affected by platform actions to be large compared to the dimensionality of the platform action, and hence $d \geq p$. In this regime, surjectivity of $h$ is a reasonable assumption and because $q_c$ is the composition of $h$ and $r$, surjectivity of $q_c$ follows. The Jacobian rank condition imposes a form of ``monotonicity'' on $q_c$. In the video recommender system setting this could correspond to: more views in category $i$ cause more recommendations in category $i$---a plausible assumption on a ML-driven system. \Cref{ass:expressive-control} is also supported by ideas proposed in \citet{Dean2019RecommendationsAU}; they suggest that recommendation systems should be designed such that users have the ability to design the recommendations they see indirectly via the actions they take. This prescription corresponds in spirit to the surjectivity condition of \Cref{ass:expressive-control}.

\subsection{General identifiability result}\label{sec:general-identifiability}

We now present our main identifiability result. The proof can be found in \Cref{proof:lem:finite-discrete-full-coverage}.

\newcommand{\umo}{z_{-1}}

\begin{theorem}
\label{lem:finite-discrete-full-coverage}
Let the dynamical system in \eqref{eqn:gen_model} have a responsive platform action. 
Let \Cref{ass:independent_noise} hold. Fix a $\bT\geq2$ and let the auditor observe $R_{\bK=1}$.
Then, 
\begin{enumerate}
\item if the system exhibits $\bM=2$ consumption shocks prior to time $\bT$, 
the steerabiltiy of consumption $\steer_\bT(u, u')$  
is identifiable for any $u, u' \in \R^p$. 
\item if the system exhibits $\bM<2$ consumption shocks prior to $T$, then for any $f, g, h, r$, there exists a distribution of $(x_{\bT-2}, u_{\bT-2})$ such that for all $u \neq u'$, the steerability of consumption $\steer_\bT(u, u')$ is unidentifiable.
\end{enumerate}
\end{theorem}

In words, this result states that consumption shocks on two preceding state variables are necessary and sufficient for the auditor to identify the steerability of consumption from observations. 
A single consumption shock is not enough for identifiability because $h$ can be a deterministic function. Thus, in this case for any given combination of of $x_t, u_{t-1}$, the auditor is only able to see one corresponding value of $u_t$, which means overlap is not satisfied. The second noise spike is necessary to provide another degree of freedom which provides enough variation for overlap, making the steerability of consumption identifiable. 
This result suggests that auditors should select $\bT$ to be a time step following the occurrence of consumption shocks; e.g., the auditor should use observations following unexpected news events or economic shocks to estimate the steerability of consumption.
We note that our analysis crucially relies on accounting for how the noise propagates through the system across multiple time steps. Because the standard causal model in \Cref{fig:general-causal-graph} is time agnostic, it is not expressive enough to make a claim like \Cref{lem:finite-discrete-full-coverage}.

The two main advantages of our approach are that
a) \Cref{ass:expressive-control} is an assumption on the design of the platform action which can be verified with enough knowledge of the platform, and b) we allow the platform action to be deterministic in its inputs, a setting which subsumes many practical ML-driven systems.
This stands in contrast to typical overlap assumptions, which are often unverifiable and de facto require explicit (and potentially unnatural) exogenous variation on the platform action.

\section{Exploiting longer rollouts for identifiability in the linear model}
\label{sec:linear-model}

In practice, an auditor may have access to longer rollouts of observations ($\bK > 1$). A natural question is whether they can exploit this information to make it easier to estimate the steerability of consumption. In this section we investigate this question in the linear setting, while we leave the general setting for future work.
More specifically, in this section, we will instantiate our model \eqref{eqn:gen_model} as follows: 
\begin{align}\label{eqn:lin_dynamics}
    \begin{split}
        f(x) &\defeq A x \qquad g(u) \defeq B u  \\
        h(x) &\defeq C x \qquad r(u) \defeq D u,
    \end{split}
\end{align}
where $A\in \R^{d, d}, B\in\R^{d, p}, C\in \R^{p, d}, D\in \R^{p,p}$.
The linear dynamics admit a clean characterization of the tradeoff between rollout length and conditions for identifiability. Linear state dynamics is certainly a strong assumption, but in has proven to be a useful approximation in control theory---e.g., quadrotors can be effectively controlled with a linear controller (e.g.,  a proportional-integral (PI) controller) relying on a linear state dynamics model \citep{BouabdallahNoSi04}.

\newcommand{\Plin}{P}
\newcommand{\hPlin}{\hat{P}}

In this linear setting, identifying the steerability of consumption reduces to identifying the matrix $B$, namely because $\steer_\bT(u,u') = B (u'-u)$.
We will again consider identifiability under consumption shocks. However, for the linear case a weaker definition suffices\footnote{To show that full-support implies full-span, apply \Cref{lem:positive-q-density} to the function $h(x) =  a^\top x$.}. 
\begin{definition}[Fully-spanning consumption shock]
\label{ass:noise-span}
We say there is a fully-spanning consumption shock at time $t$, if $\xi_t$ is such that for all vectors $a \in \R^{d}$ with $a \neq 0$, $a^\top  \noise_{t}$ is almost surely not a constant. 
\end{definition}
We will also replace the responsive platform action assumption (\Cref{ass:expressive-control}) with a full rank condition on the linear system.
\begin{definition}[Full-row-rank platform action]
\label{ass:row-rank} 
For a given $\bM \geq 2$, we say the platform has a full-row-rank platform action over a span of $\bM$ steps if $C$ and $D$ are such that the matrix $[DC, \ldots, D^{\bM-1}C]$ has full row rank. 
\end{definition}
In the linear setting, a full-row rank platform action which spans $M=2$ time steps is also an expressive platform action (\Cref{ass:expressive-control}). Similarly, an expressive platform action is also a full-row rank platform action which spans $M=2$ time steps.
\Cref{ass:row-rank} serves to generalize \Cref{ass:expressive-control} beyond the $\bK = 1$ setting of \Cref{sec:identifiability}. This generalization turns out to be the crucial piece for characterizing the benefits of observing longer rollouts, which we  formalize in the following result. The proof can be found in \Cref{proof:thm:identifiabile-five-tuple}.

\newcommand{\xmat}{X}
\newcommand{\umat}{U}
\begin{theorem}
\label{thm:identifiabile-five-tuple}
Consider the dynamical system in \eqref{eqn:gen_model} 
with linear functions $f, g, h, r$ defined in \eqref{eqn:lin_dynamics}.
Let \Cref{ass:independent_noise} hold.
Fix a time step $\bT\geq \bK+1$, let the auditor observe iid samples of $R_\bK$. Let there be a fully-spanning consumption shock at time step $\bT-\bK$. Then, 
\begin{enumerate}
\item[a)]
if $\bK =1$, then for any $A, B, C, D$, there exists a distribution over $(x_{\bT - 2}, u_{\bT-2})$ such that $\steer_\bT(u, u')$ is unidentifiable.
\item[b)] 
if $\bK \geq 2$, then full-row-rank platform action over the span of $\bK$ steps is sufficient for identifiability of  $\steer_\bT(u, u')$ for any $u, u'$.
\item[c)]
if $\bK \geq 2$, $x_{\bT - \bK - 1} = u_{\bT - \bK -1} = 0$, and $\noise_t = 0$ for $t \geq \bT - \bK + 1$, then full-row-rank platform action over the span of $\bK$ steps is necessary for identifiability of $\steer_\bT(u, u')$ for any $u, u' $.
\end{enumerate}
\end{theorem}

\Cref{thm:identifiabile-five-tuple} fully characterizes the tradeoff between identifiability, length of the observed rollout, and rank conditions on the platform dynamics matrices in the linear setting. 
Summarizing briefly,
one consumption shock is not enough to identify the steerability of consumption from only observations of $R_{\bK=1}$---just like in the general setting---but one consumption shock is enough to identify steerability of consumption from observations of $R_{\bK \geq 2}$ in the linear setting. Moreover, as $\bK$ gets larger, the rank assumptions required become easier to satisfy, allowing for more poorly conditioned dynamical systems to be identifiable. Thus, our linear dynamical system model enables us to take advantage of observing longer sequences of interactions between consumer and platform, ultimately making it easier to identify the steerability of consumption.

\newcommand{\xprv}{x_{-1}^\top }
\newcommand{\uprv}{u_{-1}^\top }

\section{Estimation from finite samples}
\label{sec:estimators}

The previous sections concerned identifiability---whether an auditor can estimate the steerability of consumption with infinite observations. In practice, the auditor will only have access to a finite number of observations. To this end, we propose two finite-sample estimators of the steerability of consumption. We introduce the \twostageregression estimator which leverages the structure of our data generation model and is reminiscent of double machine learning \citep{ChernozhukovChDeDuHaNeRo17}. This estimator can be applied if observations of $R_{\bK=2}$ are available. 
We also outline a non-parametric estimator based on the adjustment formula \Cref{eqn:adjustment-formula} that only requires observations of $R_{\bK=1}$. This estimator is also applicable to the standard causal model in \Cref{fig:general-causal-graph}, though at the cost of being less tailored to the time-aware model we propose. The analysis of the second estimator can be found in Appendix~\ref{sec:adjustment-estimator}.

\label{sec:double-ml}
\renewcommand{\loss}{\ell}

The \twostageregression estimator assumes that the auditor has iid observations of $R_{\bK=2}$. The estimator is always well defined, even when the overlap conditions needed for theoretical guarantees do not hold.
We will analyze this estimator in the linear setting from Section~\ref{sec:linear-model} and without loss of generality, we set $\bT=3$. Our results can be generalized to settings where $f, g, h, r$ are from a non-linear function class (e.g., via Rademacher complexity and VC-dimension arguments), but we focus on the simple linear setting for the sake of clarity. In particular, for the remainder of this section, assume data is generated according to the dynamical system \Cref{eqn:gen_model} with functions $f,g,h,r$ defined in \Cref{eqn:lin_dynamics}. 

We let $x_t\kth , u_t\kth $, $\noise_t\kth $ denote the $k$th observations of $x_t$, $u_t$, and $\noise_t$ respectively. Let $X_t \in \R^{d, n}$, $U_t \in \R^{p, n}$, and $E_t \in \R^{d, n}$ be matrices that comprise the $n$ samples of $x_t$, $u_t$, and $\noise_t$ respectively. 
The \twostageregression estimator is defined as $\what B$, where
\begin{align*}
    \what{C} &\defeq \argmin_{C \in \R^{p, d}} \frac{1}{2n} \lfro{U_1 - C X_1}^2 \\ 
    \what{H} &\defeq \argmin_{H \in \R^{d, d}} \frac{1}{2n} \lfro{X_2 - H X_1}^2  \\
    \what{B} &\defeq \argmin_{B \in \R^{d, p}} \frac{1}{2n} \lfro{X_3 - \what{H} X_2 - B(U_2 - \what{C} X_2)}^2.
\end{align*}
The intuition behind why this estimator works comes from the following relationship: $x_3 - (A+BC)x_2 = B(u_2 - Cx_2)$. We first estimate $H:=(A+BC)$ and $C$ using $\hat H$ and $\hat C$ respectively. Then, we regress $x_2 - \hat H x_1$ against $u_1 - \hat C x_1$ to get an estimate of $B$. Recall that knowing $B$ is sufficient to estimate the steerability of consumption $\steer_\bT(u,u')$ for any $u,u'$, as $\steer_\bT(u,u') = B(u' - u)$.

 We need the following assumption to be satisfied in order to present our convergence result for this estimator.
 
\begin{assumption}[$\rho$-Bounded System Dynamics]
\label{ass:rho-bound}
The linear dynamical system specified by \eqref{eqn:gen_model} and \eqref{eqn:lin_dynamics} has $\rho$-Bounded System Dynamics if $\opnorm{A+BC} \leq \rho {\sigma_{\min{}} (DC)}$.
\end{assumption}

To understand this assumption, consider the quantity $\frac{\ltwo{x_2 - \noise_2} }{ \ltwo{u_2}}$: this is the ratio between the magnitude of the state and platform action after one time step of evolution, ignoring noise and assuming the system starts from equilibrium $x_{0} = u_{0} = 0$. Because $\frac{\ltwo{x_2 - \noise_2} }{ \ltwo{u_2}} \leq \frac{\opnorm{A+BC} \ltwo{x_1}}{\sigma_{\min{}} (DC)\ltwo{x_1}}$, having Bounded System Dynamics ensures that the magnitude of state and platform actions are of the same scale. 
We will use the notation $\kappa_A $ to denote the condition number of a matrix $A$ and $\scov_1$  to denote the sample covariance of $\noise_1$, defined as
\[\kappa_A \defeq \frac{\sigma_{\max}(A) }{ \sigma_{\min}(A)}\quad \text{and} \quad \scov_1 \defeq \frac{1}{n}\sum_{k=1}^n \noise_1\kth (\noise_1\kth)^\top. \]

We now provide a convergence result for the \twostageregression estimator of $B$ in \Cref{thm:doubleml-linear-conv-rate}; the proof can be found in \Cref{proof:thm:doubleml-linear-conv-rate}. For simplicity, we let $\noise_3 = 0$; our analysis can be  extended to handle settings where $\noise_3 \neq 0$.

\begin{theorem}\label{thm:doubleml-linear-conv-rate}
Consider the dynamical system in \eqref{eqn:gen_model} with $x_{0} = u_{0} = 0$ and $\noise_3 = 0$, with functions $f, g, h, r$ defined in \eqref{eqn:lin_dynamics}, and with full-row-rank platform action over the span of $\bK=2$ steps.
Let the auditor observe $n$ iid samples of $R_{\bK=2}$.
Let $\E \ltwo{\noise_2}^2 = \sigma_2^2 d$, and \Cref{ass:rho-bound} hold. Let $\mc{G}$ denote the event where $X_1 X_1^\top $ is invertable. If $\E\left[\kappa_{\scov_1}^2\lambda_{\min{}} (\scov_1)^{-1}\right] \leq \tau_1$, then 
\begin{align*}
    \frac{1}{pd}\E \left[\lfro{\what{B} - B}^2 \mid \mc{G}\right] &\leq \frac{\sigma_2^2 \rho^2 \kappa_{DC}^2 \tau_1}{n}.
\end{align*}
In the case where $p=d$, if $\opnorm{\E\left[\scov_1\inv\right]} \leq \tau_2$ we have that 
\begin{align*}
    \frac{1}{d^2}\E \left[\lfro{\what{B} - B}^2\mid \mc{G}\right] &\leq \frac{\sigma_2^2 \rho^2 \tau_2}{n}.
\end{align*}
\end{theorem}

We note that rank condition on $DC$ (\Cref{ass:row-rank}) in this result is the same as the rank condition from the identifiability result in the linear setting (\Cref{thm:identifiabile-five-tuple}).
The $\tau_1,\tau_2$ conditions are a bit technical, but they essentially just require $\noise_1$ to be well behaved. 

To illustrate, consider a simple Gaussian noise example.
Suppose $\noise_1$ and $\noise_2$ are drawn iid from $\normal(0, \sigma_1^2 I_d)$ and $p=d$.  We have $\E \ltwo{\noise_2}^2 = \sigma_2^2 d$. For $n\geq d$, $\scov_1$ is almost surely invertible. 
$(X_1X_1^\top  / \sigma_1^2)\inv$ has an inverse Wishart distribution 
and thus, $\E[\scov_1\inv] = \frac{n}{(n-d-1) \sigma_1^2} I_d$ for $n > d+1$. \Cref{thm:doubleml-linear-conv-rate} gives us $\E \lfro{\what{B} - B}^2 \leq \frac{d^2\sigma_2^2 \rho^2}{(n-d-1) \sigma_1^2},$
which scales roughly like the standard linear regression error rate.

\section{Empirical investigations}
\subsection{Case study: price elasticity of demand}\label{sec:experiments}

We apply our model to the task of estimating the price elasticity of demand (PED) from time series data. Estimating the PED is an example of estimating steerability of consumption 
in the sense that we are interested in how the price (platform action) affects the demand (consumption). 
We use an avocado time series dataset \citep{kaggle-avocado} that consists of biweekly measurements of the prices of avocados and the amount of avocados purchased by region in the US from 2015 to 2018. For a week $t \in [N]$, $u_t$ corresponds to the logged average avocado price,
and $x_t$ corresponds the logged number of avocados purchased. Additional details can be found in \Cref{sec:additional-detail-exp}.
We posit the following:
\begin{align*}
    x_t = \tilde f(z_{t-1}) + g(u_{t-1}).
\end{align*}
where $z_{t-1}$ denotes the set of confounding variable that we adjust for, which we will specify shortly. In this model, the PED is defined as $\nabla g$. This quantity is a curve if the function $g$ is non-linear; however, in this section, we will assume that $g$ is linear, which reduces the problem of estimating the PED into one of estimating a scalar.

\paragraph{Varying the adjustment set to characterize overlap violations.}
Our primary focus in this section is to investigate whether the Markovian assumption on the system dynamics our model posits actually mitigates overlap violations. To do this, we vary the size of the confounding set to measure overlap violations, as well as variance and bias of different estimators.
We look at a sliding window over the data $\{(x_{t-\bK}, \ldots, x_{t}, u_{t-\bK}, \ldots, u_{t-1})\}_{t=\bK}^{N}$. We treat these samples as the iid observations of $R_\bK$ that the auditor observes. We will use $u_{t-1}$ as the treatment variable, $z_{t-1} \defeq (x_{t-\bK}, \ldots, x_{t-1})$ as the confounders, $x_{t}$ as the outcome.
We will vary $\bK$---the size of the confounding set---to explore how the size of the confounding set affects estimation.

\paragraph{Empirical setup.} We will analyze three estimators: adjustment formula estimator, random forest double ML (RF-DML), and linear regression double ML (LR-DML). The adjustment formula estimator relies on computing \eqref{eqn:adjustment-formula} on a discretized platform action and consumption variables. The discretization is important to ensure overlap over confounder and treatment variables, as the adjustment formula estimator is not well defined without overlap. 
In particular, let $\{\mc{Z}_\gamma\}_\gamma, \{\mc{U}_\beta\}_\beta$ denote discretizations of the confounders $z_{t-1}$ and platform action $u_{t-1}$, and let $\beta(u)$ be such that $u \in \mc{U}_{\beta(u)}$. We define the adjustment formula estimator as:
\begin{align*}
    \hat{x}(u) \defeq \sum_{\gamma} \left[ \frac{\sum_t  x_{t+1}\bindic{z_t \in \mc{Z}_\gamma, u_t \in \mc{U}_{\beta(u)} }}{\sum_t \bindic{z_t \in \mc{Z}_\gamma, u_t \in \mc{U}_{\beta(u)} }} \right] \frac{\sum_t \bindic{z_t \in \mc{Z}_\gamma }}{n}.
\end{align*}
Detailed discussion and theoretical guarantees regarding the adjustment formula can be found in \Cref{sec:adjustment-estimator}.
We discretize the logged price into two buckets: $(-0.479, 0.131], (0.131, 0.683]$ and the logged demand into two buckets $(14.539, 15.014], (15.014, 15.837]$. After using the adjustment formula estimator to estimate the effect price has on demand, we then use this estimator to assign predicted demands to all of the prices observed in the dataset. We then use linear regression to estimate the slope of the relationship between predicted demand and price---this is what we refer to as the adjustment formula estimate of the PED. This approach is motivated by methods suggested by \citet{PetersenPoGrWaVa12}. The double machine learning approach \citep{ChernozhukovChDeDuHaNeRo17} first uses half of the training data to residualize the confounders out of the treatment and effect. For LR-DML, the residualizing procedure uses linear regression; for RF-DML, the residualizing procedure uses a random forest model. Then, in the second step, both RF-DML and LR-DML use the other half of the training data to perform a slightly modified version linear regression---discussed in \citet{ChernozhukovChDeDuHaNeRo17}---on the residualized treatment and residualized effect. The slope of this estimated line is the estimated PED.

\begin{table}[t]
\begin{center}
\resizebox{0.7\columnwidth}{!}{
\begin{tabular}{@{}llllll@{}}
\toprule
                     & \thead{Price\\ (Intervention $u$)} & \thead{Estimated Effect\\ on demand} & \thead{Fraction of\\ undefined terms} & \thead{Probability mass\\ of undefined terms} \\ \midrule
\multirow{2}{*}{\bK=1} &   High    & 14.95       &        0 / 2          &    0.0\%    \\
                     &  Low      &  15.11     &        0 / 2         &          0.0\%          \\ \midrule
\multirow{2}{*}{\bK=3} &  High     & 14.96 
                                            &       0 / 8           &     0.0\%\\           
                     &  Low  & 15.11 
                                            &        0 / 8          &       0.0\%\\ 
                                            \midrule
\multirow{2}{*}{\bK=5} &  High     & N/A    
                                            &       5 / 31           &     4.6\%\\              
                     &  Low  & 15.10      
                                            &        0 / 31          &         0.0\%\\ \midrule
\multirow{2}{*}{\bK=7} &  High     & N/A   
                                            &       36 / 89           &     15.1\%\\             
                     &  Low  & N/A     
                                            &        16 / 89          &         9.0\%\\ \midrule 
\multirow{2}{*}{\bK=9} &  High     & N/A    
                                            &       73 / 145           &     25.9\%\\              
                     &  Low  & N/A   
                                            &        40 / 145          &         19.7\%\\ \midrule                                                
\end{tabular}}
\end{center}
\caption{Adjustment formula estimated effects on avocado demand for price interventions. }
\label{fig:causal-eff-ped-table}
\end{table}

\paragraph{Importance of shrinking adjustment set for overlap.}
We report what the adjustment formula estimator estimates for a discretized treatment $u$ in \Cref{fig:causal-eff-ped-table}. A ``Low'' price  in the treatment column corresponds to the logged price bucket $(-0.479, 0.131]$. A ``High'' price corresponds to $(0.131, 0.683]$. 
The ``Fraction of undefined terms'' column corresponds to the number of $\gamma$ values  where $\sum_t \bindic{z_t \in \mc{Z}_\gamma} > 0$ and $\sum_t \bindic{z_t \in \mc{Z}_\gamma, u_t \in \mc{U}_{\beta(u)} } = 0$ over the total number of values of $\gamma$ where $\sum_t \bindic{z_t \in \mc{Z}_\gamma} > 0$. If ``Fraction of undefined terms'' is non-zero, then $\hat x(u)$ is not well defined. $N/A$ denotes when this occurs.
The entries of ``Probability mass of undefined terms'' column is equal to $\sum_{\gamma} \frac{\sum_t \bindic{z_t \in \mc{Z}_\gamma}}{n} \bindic{\sum_t \bindic{z_t \in \mc{Z}_\gamma, u_t \in \mc{U}_{\beta(u)} } = 0 }$. We can see that as $\bK$ gets larger, the number of undefined estimates, the relative fraction of undefined values, and the mass of said values gets larger. This preliminary analysis already suggests that there are overlap issues as $\bK$ gets larger.

\paragraph{Effect of shrinking adjustment set on estimator variance.}
Next, we bootstrap the adjustment formula estimator and two double ML estimators.
We find that the number of confounders heavily affects the 
bootstrapped variance of the PED estimators, suggesting that the Markovian modeling assumption (i.e., setting $\bK=1$) used and by our theory is also useful in practice. 
We report the predicted PED for all of the estimators in \Cref{fig:avocado-PED} (left). For each estimator, we bootstrap the dataset 40 times to form confidence intervals. We report the standard deviation of the bootstrapped estimates in \Cref{fig:avocado-PED} (right). We see that the variance of the adjustment formula estimator increases as the number of confounders increases. 
The RF-DML and LR-DML variance curves are fairly stable with respect to $\bK$, suggesting that our Markovian assumption does not affect the variance of those estimators by much.

\begin{figure}[t]
    \centering
    \includegraphics[width=0.43\linewidth]{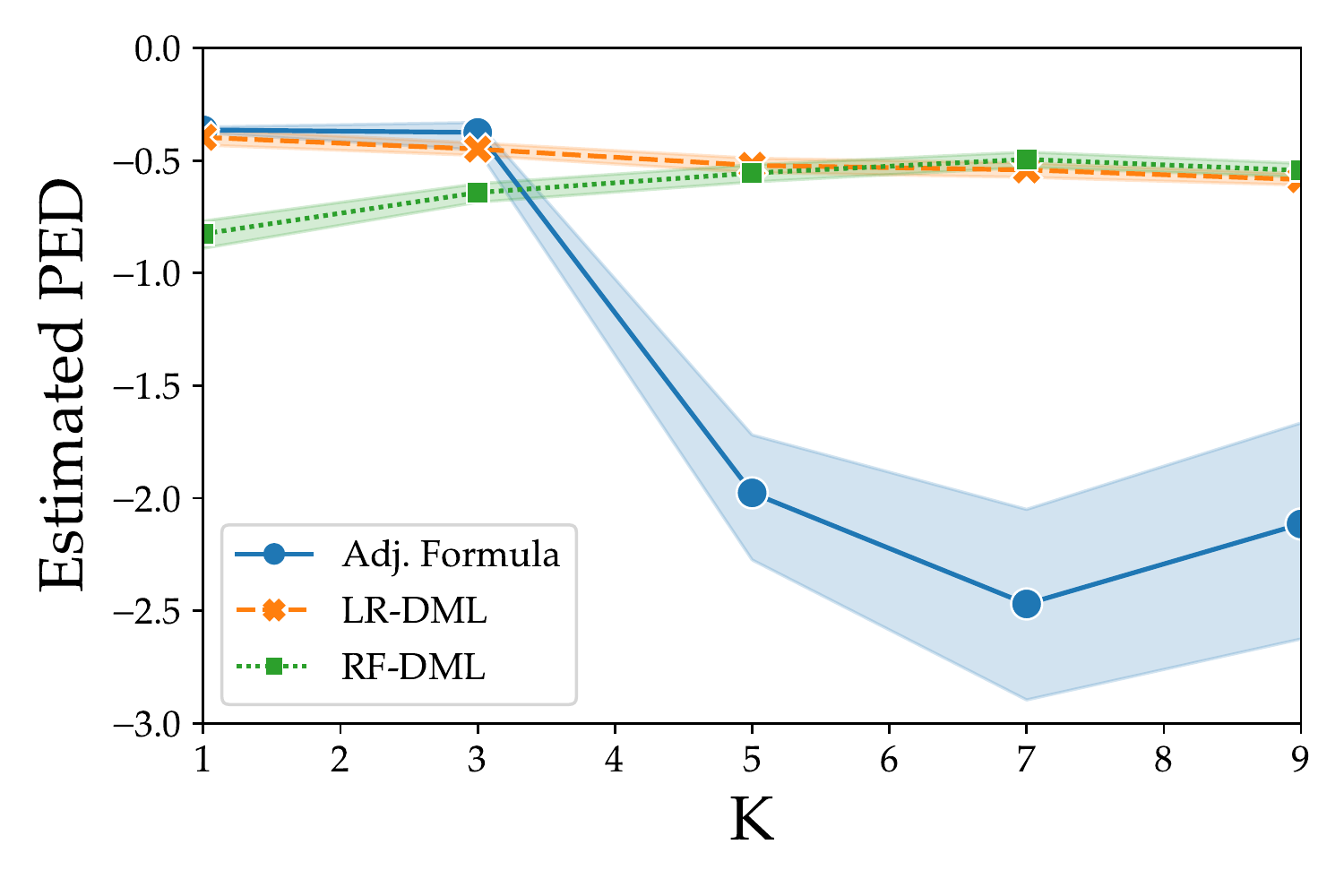}
    \includegraphics[width=0.43\linewidth]{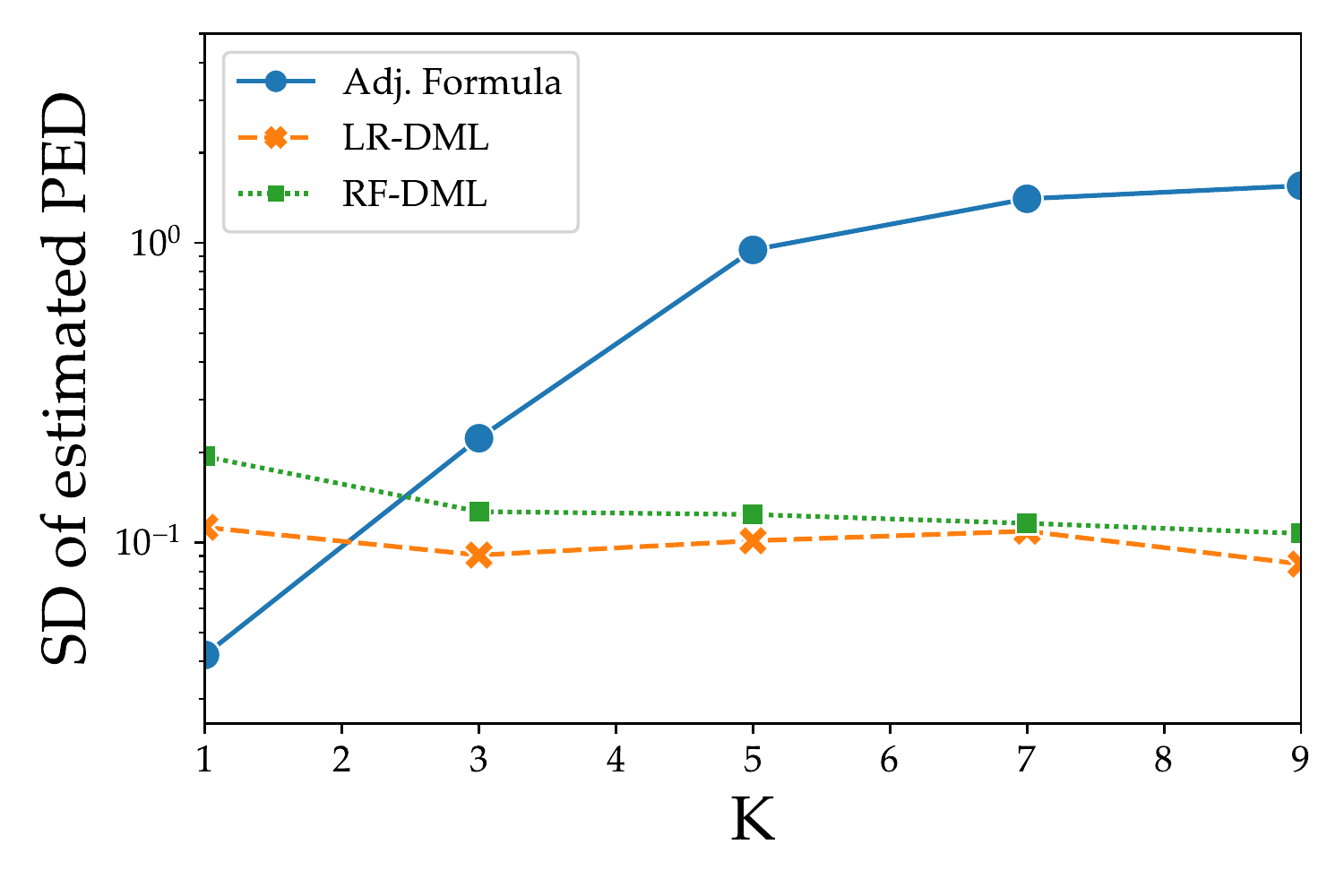}
    \caption{(left) Bootstrapped estimates of PED with 95\% confidence intervals. (right) Standard deviation of each estimator. }
    \label{fig:avocado-PED}
\end{figure}

\paragraph{Effect of shrinking adjustment set on estimator bias.}
Stronger assumptions enable identifiability, but they come at a price of potential modeling errors. We have motivated our Markovian assumption theoretically, and now we want to understand how well they reflect reality.
We use the bootstrapping technique proposed by \citet{PetersenPoGrWaVa12} for testing the bias of our estimators, which we describe now. Let $\Psi$ be the estimator of the PED we are testing, and $\Psi_a$ be the adjustment formula estimator of the PED. Further, let $y^\bK$ denote the avocado dataset for sequences of length $\bK$, and let $Y^\bK$ denote a bootstrapped sample constructed from $y^\bK$. We plot an empirical estimate of
\begin{align}\label{eqn:bootstrap-bias}
    \E[\Psi(Y^\bK)] - \Psi_a(y^\bK)
\end{align}
using 40 bootstrap samples with confidence intervals in \Cref{fig:avocado-PED-bias} (left). We see that the adjustment formula and LR-DML estimators have small bias for small values of $\bK$, and all estimators have larger bias for large values of $\bK$. 
We also plot the estimated bias defined using \eqref{eqn:bootstrap-bias} but with $\Psi_a$ replaced with $\Psi$ instead. We see that the bias still increases as $\bK$ gets larger, suggesting that more confounders also increases the bias of the estimator.

\begin{figure}[t]
    \centering
    \includegraphics[width=0.43\linewidth]{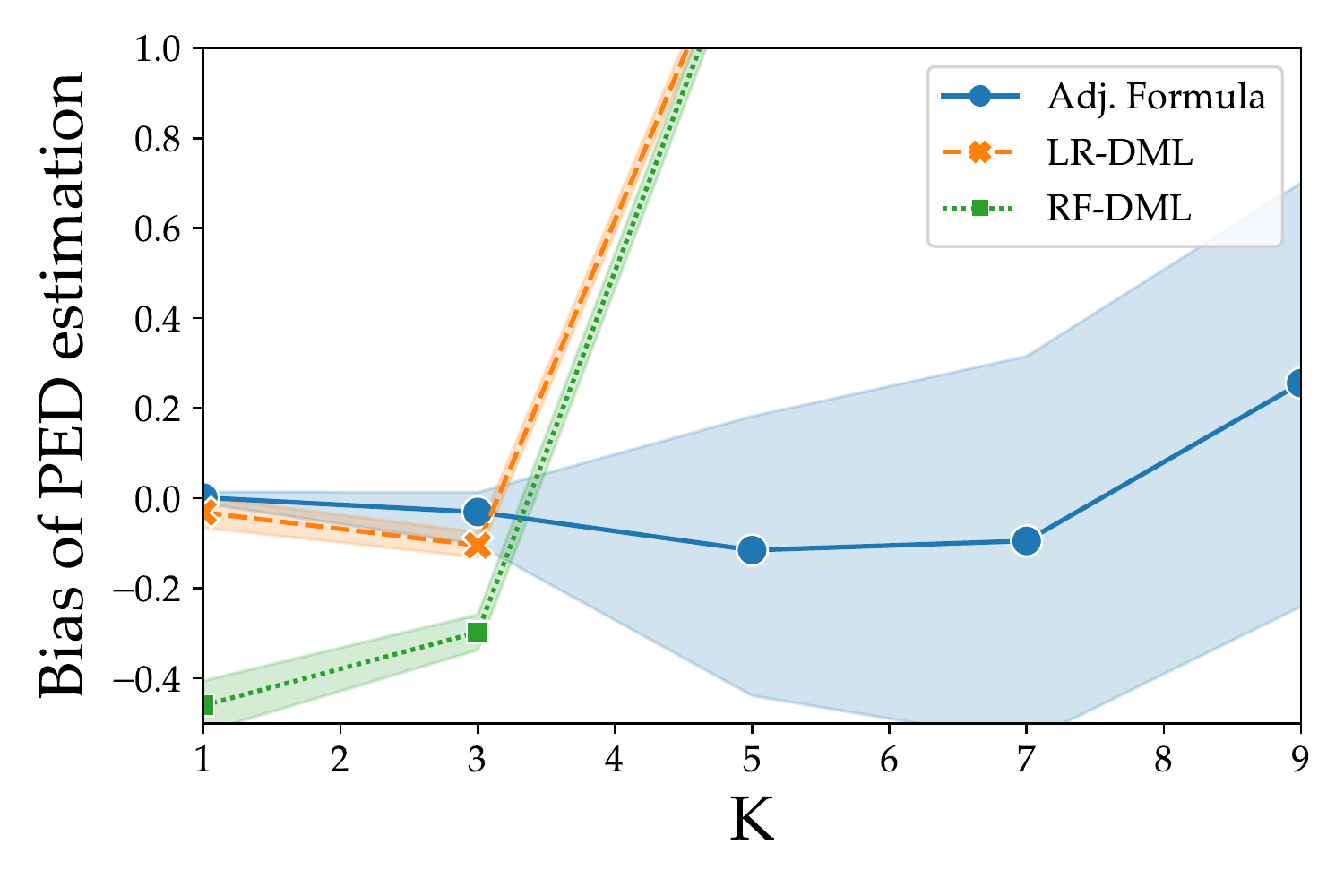}
    \includegraphics[width=0.43\linewidth]{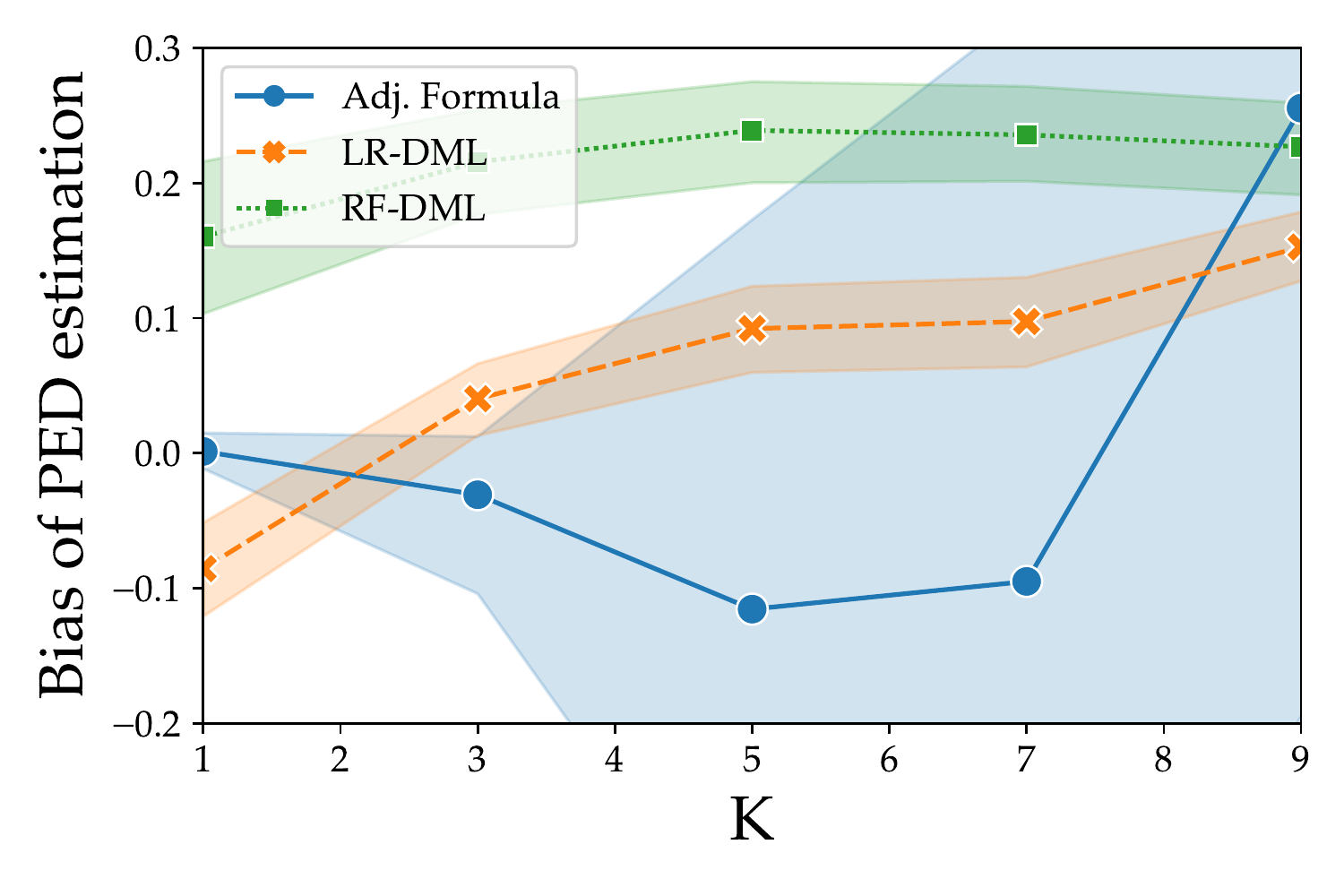}
    \caption{(left) Bias defined in \eqref{eqn:bootstrap-bias}. (right) Bias defined in \eqref{eqn:bootstrap-bias} with $\Psi_a$ replaced with $\Psi$. }
    \label{fig:avocado-PED-bias}
\end{figure}

Our experiments suggest that our Markovian assumption (i.e., $\bK=1$) does mitigate overlap issues while still accurately modeling reality. We believe the increase (with $\bK$) in bias and variance of the estimators is caused by overlap issues; as $\bK$ gets larger, the dimension of the confounders gets larger, making overlap harder to satisfy.

\begin{figure}[t]
    \centering
    \includegraphics[width=0.3\linewidth]{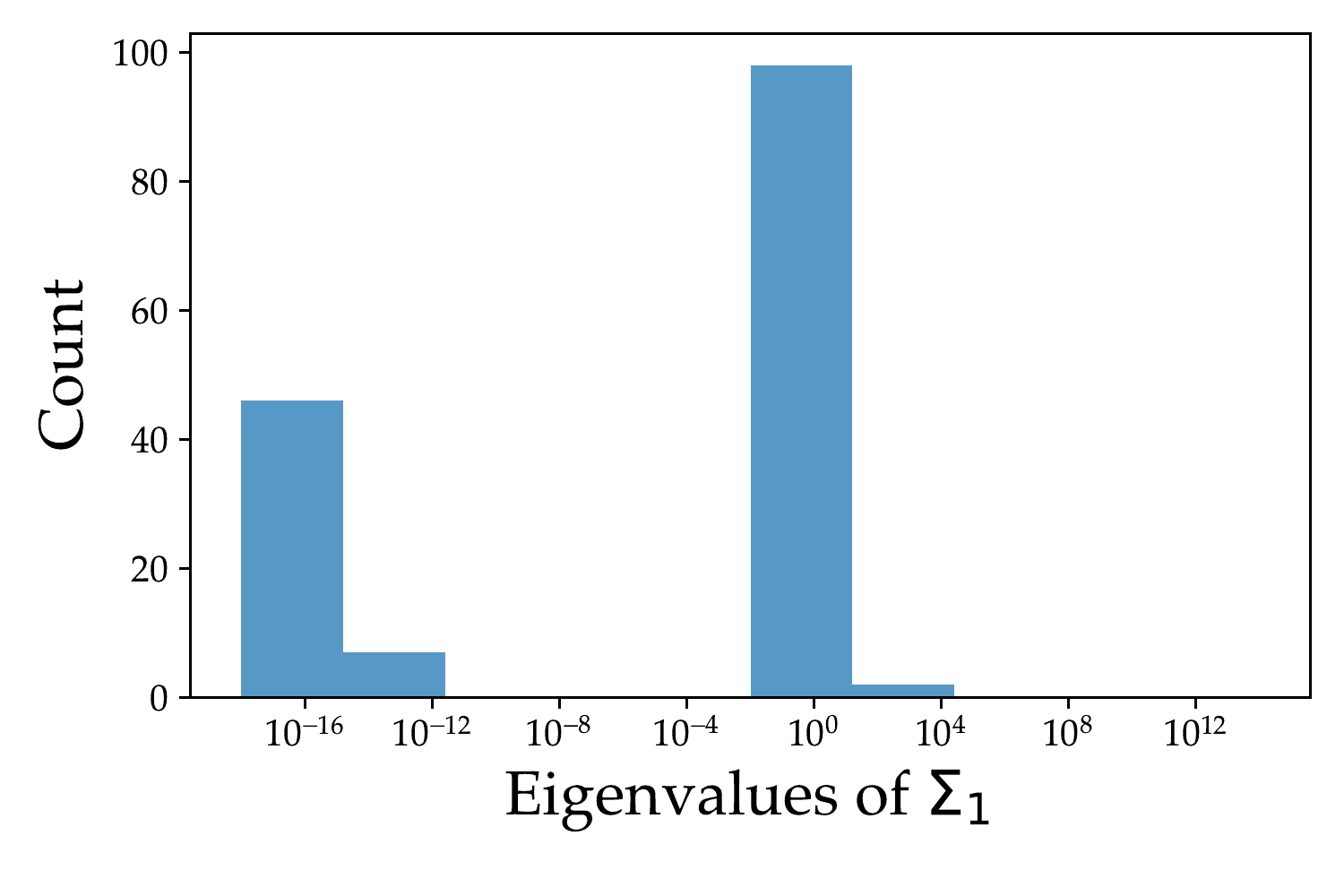}
    \includegraphics[width=0.3\linewidth]{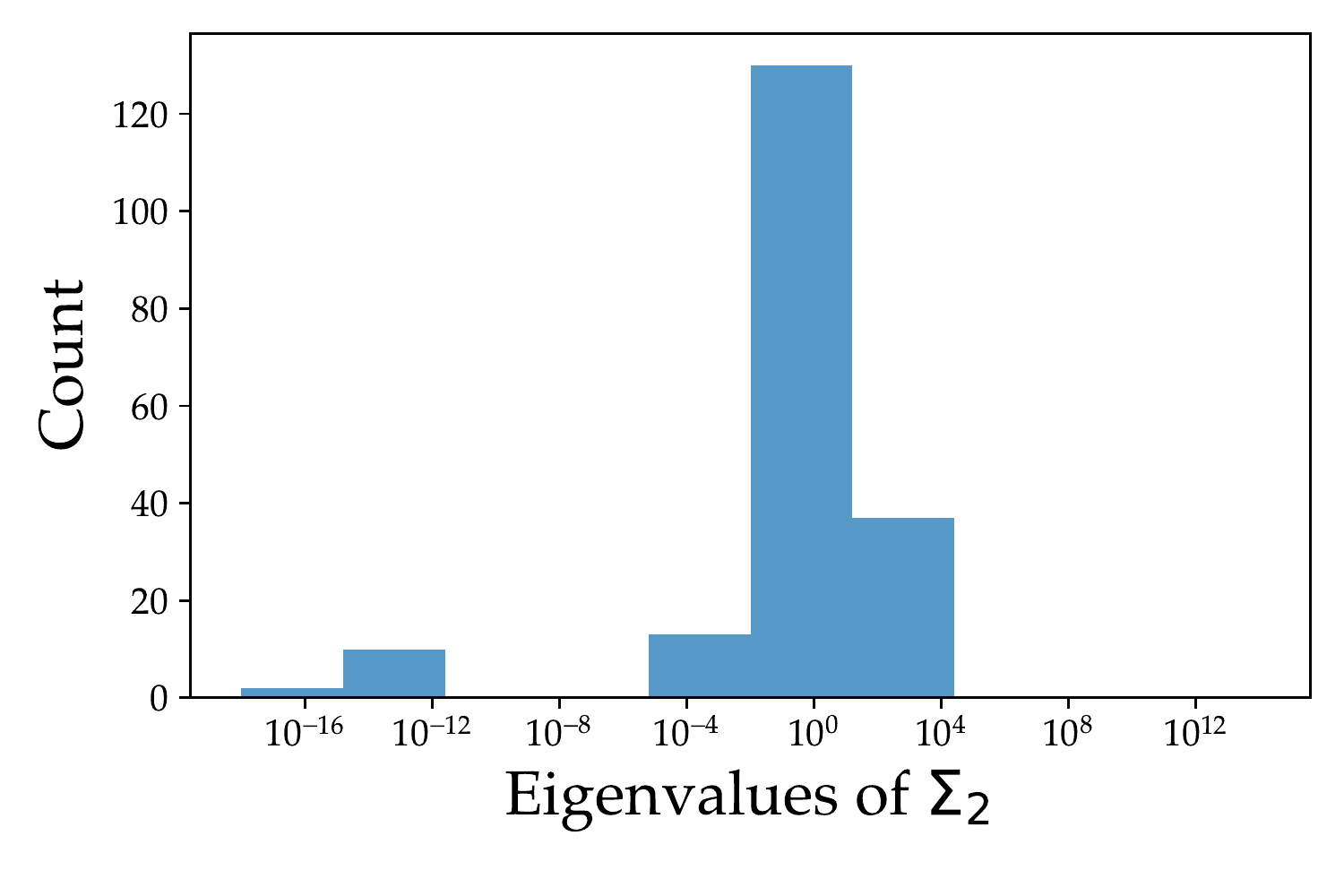}\\
    \includegraphics[width=0.3\linewidth]{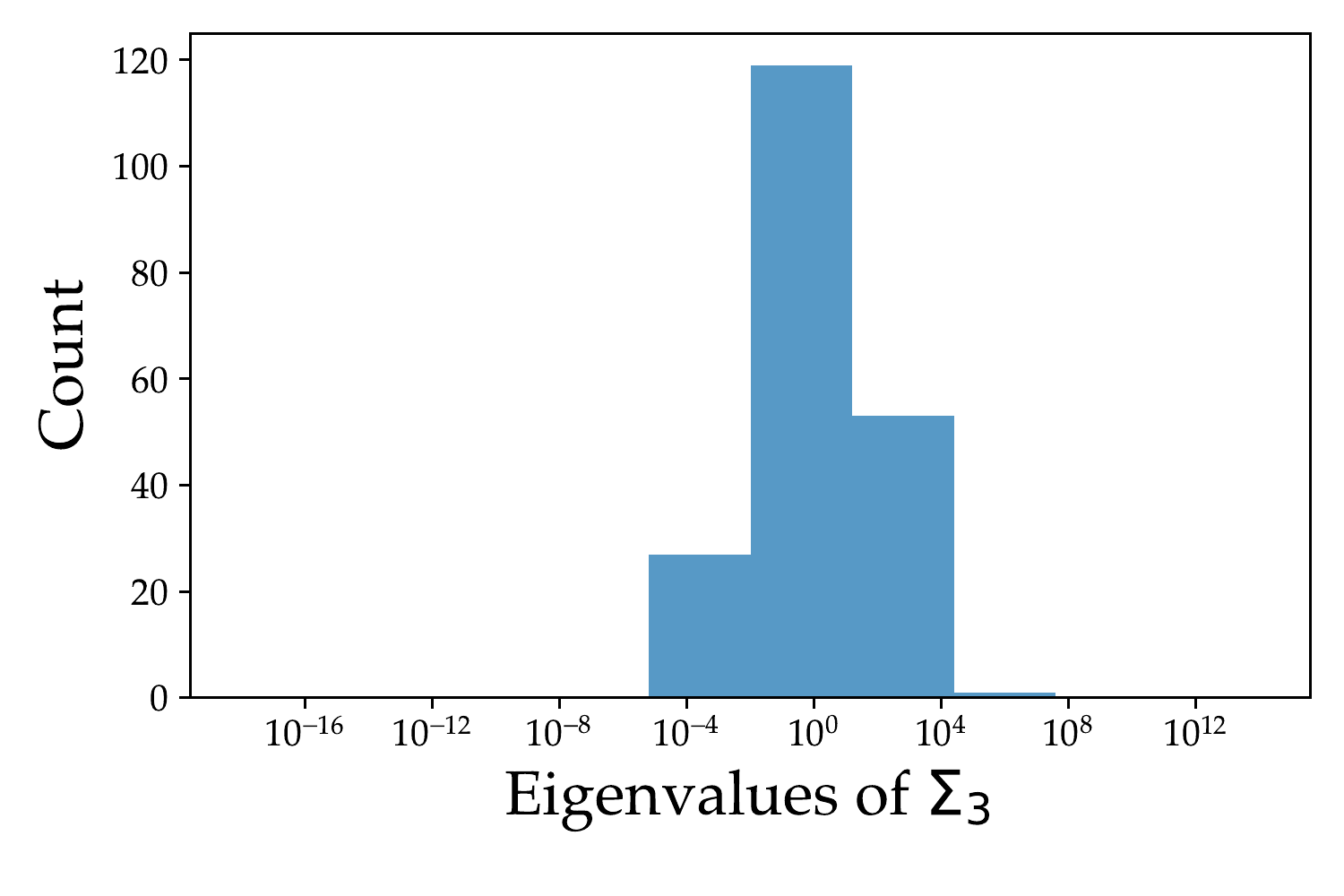}
    \includegraphics[width=0.3\linewidth]{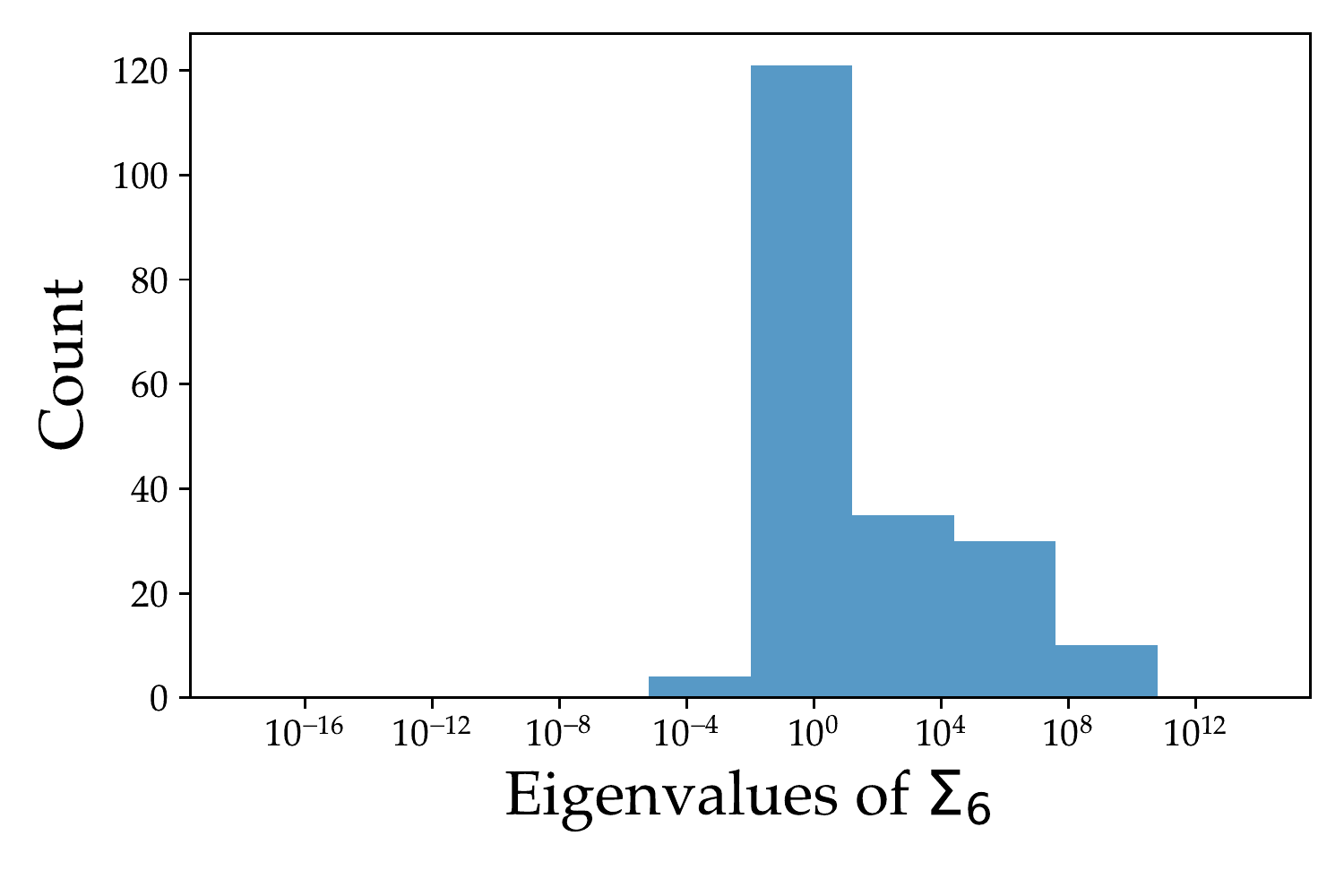}
    \caption{Histograms of eigenvalues of $\Sigma_t$ as defined in \Cref{sec:synthetic_experiments}.}
    \label{fig:synthetic_spectrum}
\end{figure}

\subsection{Synthetic experiments}

\label{sec:synthetic_experiments}
We analyze the dynamical system \eqref{eqn:gen_model} with linear dynamics \eqref{eqn:lin_dynamics} with independent Gaussian noise acting as consumption shocks $\noise$ on the states. We show how the conditioning of the problem evolves over time and how the presence of more consumption shocks in past time steps makes the steerability of consumption easier to estimate.

We let $\noise_t \eqd \normal(0, I)$ for all $t$, starting from $x_{0} = u_{0} = 0$. We consider the symmetric case where $d = p$. To generate $B$, we sample a random matrix $W$ in $\R^{d, n}$ for $n \gg d$ with 
independent standard Gaussians as its entries, and we set $B = W W^\top  / n$. We repeat this process to generated $A$ and $D$. This way of generating our dynamics matrices ensures the matrices are well conditioned. We generate $C$ the same except by instead setting $W \in \R^{d, r}$ for $r < d$, making $C$ rank $r$ instead of rank $d$. We set $d = 100$, $n = 2000$, and $r = 80$. 

For this system, we can explicitly write down how the covariance matrix of $(x_t, u_t)$, denoted $\Sigma_t$, evolves. Namely, from the dymanics
\begin{gather*}
    \begin{bmatrix}
    x_t\\
    u_t
    \end{bmatrix}
    =  J
    \begin{bmatrix}
    x_{t-1}\\
    u_{t-1}
    \end{bmatrix}
    +
    M
    \varepsilon_t \quad \text{where} \\
    \quad
    J \defeq \begin{bmatrix}
    A & B\\
    CA & CB+D
    \end{bmatrix}\qquad
    M \defeq \begin{bmatrix}
    I\\
    C
    \end{bmatrix}.
\end{gather*}
we can deduce that
\begin{align*}
    \Sigma_t = J \Sigma_{t-1} J^\top  + MM^\top  = \sum_{k=0}^{t-1}(J^k)MM^\top  (J^k)^\top .
\end{align*}
We now plot the histogram of the eigenvalues of $\Sigma_t$ for a random system that we generated. The important observable to look out for is whether $\Sigma_t$ is full rank.
Indeed, the steerability of consumption---in this case $B$ because the system is linear---is identifiable from observations of $(x_t, u_t, x_{t+1})$ if and only if $\Sigma_t$ is full rank.
To see why this is true, suppose $\Sigma_t$ is low rank and let $v$ be in the null space of $\Sigma_t$. Letting $G \defeq [A, B]$, $z \defeq [x_t^\top, u_t^\top]^\top$, and $\mathbf{1}$ denote the all one's vector of appropriate dimension, we have that
\begin{align*}
    x_{t+1} = G z = (G + \mathbf{1} v^\top) z -  \mathbf{1} v^\top z \eqd \ (G + \mathbf{1} v^\top) z,
\end{align*}
meaning that $G$ and $G + \mathbf{1} v^\top$ could have both generated the distribution observed. If $\Sigma_t$ is full rank, then linear regression will be able to recover $B$.

We note that in this system, $\rank DC = 80 =r < d$ and $\rank C = 80=r < d$. 
We believe there is an equivalence between the system in this section and the system from \Cref{thm:identifiabile-five-tuple} because of linearity, even though the settings are different---one consumption shock and full observation of each rollout (i.e., observations $R_\bK$) in the theory versus multiple  consumption shocks and one timestep of observation (i.e., observations of $R_{\bK = 1}$) in this section. We are not able to prove this equivalence, but we provide some empirical evidence supporting this conjecture.
\Cref{thm:identifiabile-five-tuple} suggests that observing $(x_2, u_2, x_3)$ is not sufficient for identifiability, as $\rank DC$ is not full row rank. This is consistent with the eigenvalue histogram of $\Sigma_2$ in \Cref{fig:synthetic_spectrum} as there are still $0$ eigenvalues. 
However, since in this system $\rank [DC, D^2C] = 100 =d$,
\Cref{thm:identifiabile-five-tuple} suggests that observing $(x_3, u_3, x_4)$ is sufficient for identifiability. This is also consistent with the eigenvalue histogram of $\Sigma_3$ in \Cref{fig:synthetic_spectrum}, as all eigenvalues are bounded away from $0$ at that time step. Moreover, we see that the eigenvalues of get larger as more time passes: e.g., the eigenvalue mass of $\Sigma_6$ is further to the right of the eigenvalue mass of $\Sigma_3$ in \Cref{fig:synthetic_spectrum}. This suggests that more noise spikes over more time steps make the observations better conditioned, likely making estimating the steerability of consumption easier for the auditor to estimate in practice; e.g., the condition number terms in \Cref{thm:doubleml-linear-conv-rate} will be smaller.

\section*{Acknowledgements}

The authors would like to thank Michael M\"uhlebach for stimulating discussions on the project, and Saminul Haque for helpful technical discussions surrounding \Cref{lem:positive-q-density}. This work was supported by the T\"ubingen AI Center. Gary Cheng acknowledges support from the Professor Michael J. Flynn Stanford Graduate Fellowship.

\bibliography{ more-bib}
\bibliographystyle{plainnat}

\newpage
\appendix

\section{Relaxing \Cref{ass:independent_noise}} 
\label{sec:independence-discussion}
In context of \Cref{lem:finite-discrete-full-coverage}, we can replace \Cref{ass:independent_noise} with the following weaker assumption
\begin{assumption}
Let $\noise_\bT$ be such that $\E[\noise_\bT] = \E[\noise_\bT \mid u_1 = u, x_1 =z]$
for all $u \in \R^p$ and $z \in \R^d$.
\end{assumption}
This ``no-correlation'' type assumption is required for showing admissibility (\Cref{lem:admissibility}), and it only needs to apply to the exogenous variation affecting the time step $\bT$ we are interested in estimating $\steer_\bT$.
Having said that, \Cref{ass:independent_noise} is necessary for \Cref{thm:identifiabile-five-tuple}. Mutual independence is crucial for our proof technique.

\section{Additional Experimental Details in support of \Cref{sec:experiments}}
\label{sec:additional-detail-exp}

The avocado time series dataset is comprised of several time series spanning different regions of the United States. To construct the dataset we are operating on, we combine data from two regions---Southeast and Great Lakes---chosen by pricing and demand similarity.

For the bootstrapping experiments shown in \Cref{fig:avocado-PED,fig:avocado-PED-bias}, we have modified the adjustment formula estimator to silently fail when overlap does not hold. In particular, unlike our results in \Cref{fig:causal-eff-ped-table}, for terms in the adjustment formula $\hat{x}$ (defined in \Cref{sec:experiments}) where $\sum_t \bindic{z_t \in \mc{Z}_\gamma} > 0$ and $\sum_t \bindic{z_t \in \mc{Z}_\gamma, u_t \in \mc{U}_{\beta(u)} } = 0$, we set $ \frac{\sum_t  x_{t+1}\bindic{z_t \in \mc{Z}_\gamma, u_t \in \mc{U}_{\beta(u)}}}{\sum_t \bindic{z_t \in \mc{Z}_\gamma, u_t \in \mc{U}_{\beta(u)}}}$ equal to $0$.
This modification could cause the adjustment formula estimator to underestimate the PED for large $\bK$, potentially causing the bias to spike for larger $\bK$ for LR-DML and RF-DML. This is why we also plot the estimated bias defined using \eqref{eqn:bootstrap-bias} but with $\Psi_a$ replaced with $\Psi$ instead in \Cref{fig:avocado-PED-bias}. We see that the bias still increases as $\bK$ gets larger, suggesting that more counfounders does in fact increases the bias of the estimator.

\section{Adjustment formula estimator}
\label{sec:adjustment-estimator}
\newcommand{\familyx}{\mc{N}}
\newcommand{\familyu}{\mc{M}}
\newcommand{\setx}{\mc{X}}
\newcommand{\setu}{\mc{U}}
\newcommand{\idxx}{\alpha}
\newcommand{\idxu}{\beta}
\newcommand{\subidxx}{_{\idxx(x)}}
\newcommand{\subidxu}{_{\idxu(u)}}
\newcommand{\ds}{\mc{D}^n}
\newcommand{\constone}{c_1}
\renewcommand{\xdom}{\cup \familyx}
\renewcommand{\udom}{\cup \familyu}
\newcommand{\Zest}[1]{\what P(x_1 \in #1)}

Admissibility of the dynamical system we are studying (\Cref{lem:admissibility}) makes estimating the adjustment formula (\Cref{def:adjustment-formula}) sufficient for estimating the steerability of consumption.
Since $x$ and $u$ can take on continuous values we start with discretizations of $\R^d$ and $\R^p$ denoted as finite collections of bounded, non-intersecting sets $\familyx \defeq \{\setx_{\idxx}\}$ and $\familyu \defeq \{\setu_{\idxu}\}$ indexed by $\idxx$ and $\idxu$ respectively.
Suppose that every element of $\familyx$ and $\familyu$ has diameter at most $\varepsilon / 2$ and has Lebesgue measure greater than 0. For a point $x\in \cup \familyx$, define $\idxx(x)$ such that $x \in \setx\subidxx$. Define $\idxu(u)$ respectively. 
We will assume we have $n$ samples of the form $\ds=\{(x_1\kth, u_1\kth, x_2\kth)\}_{k=1}^n$,
where every sample is drawn iid from \eqref{eqn:gen_model}. With these quantities, we form estimates of the components of the adjustment formula; here without loss of generality, we set $\bT = 2$.
\begin{align*}
    \what\E[x_2 \mid u_1 \in \setu, x_1 \in \setx] &\defeq \frac{\sum_{k \in [n]} x_2\kth \bindic{u_1\kth \in \setu , x_1\kth \in \setx}}{\sum_{k \in [n]} \bindic{u_1\kth \in \setu , x_1\kth \in \setx}}\\
    \what P(x_1 \in \setx) &\defeq \frac{1}{n} \sum_{k=1}^n \bindic{x_1\kth \in \setx}.
\end{align*}
After combining, we have an estimate of the steerability of consumption:
\begin{align*}
    \hat{x}_2(u)
    \defeq \sum_{\idxx}  \what\E[x_2 \mid u_1 \in \setu\subidxu, x_1 \in \setx_\idxx]  \what P(x_1 \in \setx_\idxx).
\end{align*}

We will need some mild assumptions to prove a guarantee on the estimator. Our first assumption controls how much previous user state and platform actions affect future state actions. The magnitude of the effect must be bounded in proportion to the inputs. 

\begin{assumption}\label{ass:lipschitz-conditional-ev}
The relationship between $x_2$ and $x_1,u_1$ is $L$-Lipschitz continuous in the sense that for any $w,w'\in\R^d$ and $u,u'\in\R^p$, and with $v \defeq [u^\top , w^\top ]^\top $, it holds that
\begin{align*}
    \norm{\E[x_2 | u_1 = u, x_1 = w] - \E[x_2 | u_1 = u', x_1 = w']} 
    \leq L \norm{v - v'}.
\end{align*}
\end{assumption}

We also need to control how far the discretized conditional expectation 
$\E[x_2 \mid u_1 \in \setu, x_1 \in \setx]$
deviates from $\E[x_2|u_1 =u, x_1=x]$. To do this, we impose a regularity condition on the conditional distribution.

\begin{assumption}\label{ass:conditional-density-control}
Let $w, w'\in \R^d$ and $u, u' \in \R^p$, and with $v \defeq [u^\top , w^\top ]^\top $ be such that $\norm{v - v'} \leq \varepsilon$.
Then, for any $x\in\mathcal \cup \familyx$, the following condition on the density $p$ holds for some $\eta(\varepsilon) \in (0, 1)$ such that $\lim_{\varepsilon \to 0} \eta(\varepsilon) = 0$:
\begin{align*}
    1-\eta(\varepsilon) \leq \frac{p(u_1 = u, x_1 =w | x_2 = x)}{p(u_1 = u', x_1 =w' | x_2 = x)} \leq 1+\eta(\varepsilon).
\end{align*}
\end{assumption}

This assumption ensures that the conditional distribution is ``stable'' in any $\varepsilon$-neighborhood. Finally, we need one more assumption which guarantees we obtain enough samples for every slice of data. \Cref{ass:enough_samples} is defined with respect to the variables: cover granularity $\varepsilon >0$, error tolerances $\delta \in (0, 1)$ and $\gamma >0$, and failure probability tolerance $\rho \in (0, 1)$.
\begin{assumption}\label{ass:enough_samples}
    Let $n_{\setu, \setx} \defeq \sum_{k \in [n]} \bindic{u_1\kth \in\setu, x_1\kth \in \setx}$. Let $n_{\setu, \setx} \geq \frac{2d\sigma^2}{\gamma^2} \log(4 |\familyx| / \rho)$ for all $\setx \in \familyx$ and $\setu \in \familyu$. Further let $n \geq \max_{\setx \in \familyx} \frac{1}{2 \delta^2 P(x_1 \in \setx)^2}\log(4 |\familyx| / \rho)$.
\end{assumption}

 We present our convergence result now in \Cref{thm:adjustment-estimator-nonasymptotic}.

\begin{theorem}\label{thm:adjustment-estimator-nonasymptotic}
Consider the dynamical system in \eqref{eqn:gen_model} with any arbitrary $P_{-1}$. Let the auditor observe $n$ iid samples of $(x_1, u_1, x_2)$. 
Suppose $x_2$ is $\sigma^2$-subgaussian conditioned on $u_1$ and $x_1$.
Let $\E[\noise_2|x_1= x, u_1= w] =0$, $\E[\norm{\noise_2} \mid x_1 =x, u_1=w] \leq \constone$ for all $x \in \R^d$ and $w \in \R^p$. Let $f$ and $g$ be continuous functions, and define $\diam$ such that $\sup_{x \in \xdom, w \in \udom }\max\left( \norm{f(x)}, \norm{g(w)}\right)  \leq \diam$.
Let the conditions of \Cref{lem:finite-discrete-full-coverage} hold, \Cref{ass:lipschitz-conditional-ev} hold with $L$, \Cref{ass:conditional-density-control} hold with $\eta$, and \Cref{ass:enough_samples} hold.
For any specified $u \in \cup \familyu$ with probability at least $1 - \rho$, the following holds
\begin{align*}
    \norm{\hat{x}_2(u)  - \E[x_2 | do(u_1 \defeq u)]} 
    & \leq \delta\gamma + 2\delta \constzero \diam + \gamma  + \frac{2\eta(\varepsilon)}{1 - \eta(\varepsilon)} \left(
    2 \constzero \diam 
    + \constone \right)\\
    &\quad + L\varepsilon + \constzero \E[\norm{f(x_1)} \bindic{x_1 \not\in \cup \familyx}] + \constzero(1 - P_{x_1}(\cup \familyx)) \diam.
\end{align*}
\end{theorem}

The proof of \Cref{thm:adjustment-estimator-nonasymptotic} can be found in \Cref{proof:thm:adjustment-estimator-nonasymptotic}.
Let us go through all the terms in the bound, to verify that they can all be made arbitrarily small (with sufficient samples). $\delta$ and $\gamma$ can be made smaller, so long as the auditor receives proportionally enough samples. The auditor can create a finer discretization to make $\varepsilon$ smaller and therefore $\eta$ smaller as well. If we assume that $\E[\norm{f(x_1)}]\leq \infty$, then the last two terms tend to zero as the auditor's approximation of $\R^d$---i.e., $\cup \familyx$---covers more of the space.

\renewcommand{\xdom}{\cup \familyx}
\renewcommand{\udom}{\cup \familyu}

\newpage
\section{Proofs}
\subsection{Auxiliary results}
\label{app:poof3}
\begin{lemma}[Multivariate change of variables]\label{lem:change-of-variables}
Let $X$ be a random variable with density $p_X$ and let $Y = g(X)$ where $g$ is an invertible mapping with Jacobian $J_g$, then $p_Y(a) = p_X(g\inv(a)) |J_g(a)|\inv$.
\end{lemma}
\begin{proof}
\begin{align*}
    P(Y \in A) &= P(X \in g\inv(A)) = \int_{g\inv(A)} p_X(x) dx = \int_{A} p_X(g\inv(x)) |J_{g\inv}(x)|dx \\
    &=  \int_{A} p_X(g\inv(x)) |J_{g}(x)|\inv dx.
\end{align*}
The definition of density gives the result.
\end{proof}

\begin{definition}[Lusin's (N) condition]\label{def:lucin}
A function $f: \R^d \to \R^p$ satisfies Lusin's (N) condition if for every Lebesgue-measure 
$0$ set $A \subset \R^d$, $f(A)$ has Lebesgue-measure $0$.
\end{definition}

\begin{definition}[Non-singular measurable transformation]\label{def:nonsingular}
A function $f: \R^d \to \R^p$ is a non-singular measurable transformation if for every Lebesgue-measure $0$ set $A \subset \R^p$, the preimage of $A$, $f\inv(A)$ has Lebesgue-measure $0$.
\end{definition}

\begin{lemma}\label{lem:positive-q-density}
    For a measurable function $h: \R^d \to \R^p$, let $h\inv$ denote the preimage. 
    Let $h$ be a non-singular measurable transformation which satisfies Lusin's (N) condition.
    Let $X$ be a $\R^d$-valued random variable with measure $P_X$ and density $p_X$, and let $Y \defeq h(X)$ be a $\R^p$-valued random variable. Then the following is true:
    \begin{enumerate}
        \item $P_Y$ has a density $p_Y$ with respect to the Lebesgue measure.
        \item if $p_X(a) > 0$ for almost all $a \in \R^d$ with respect to the Lebesgue measure, then $p_Y(b) > 0$ for almost all $b\in \R^p$ with respect to the Lebesgue measure.
    \end{enumerate}
\end{lemma}
\begin{proof}
    Recall that a $\sigma$-finite measure $\nu$ has a density with respect to $\sigma$-finite measure $\mu$ if and only if $\nu$ is absolutely continuous with respect to $\mu$ (denoted as $\nu \ll \mu$).

    We prove the first point first. We will show that the measure of $Y$, $P_Y$, is absolutely continuous with respect to the Lebesgue measure $\lambda$.  Let $A \subset \R^p$ be such that $\lambda(A) = 0$, then
    \begin{align*}
        \lambda(A) = 0 \implies \lambda(h\inv(A))= 0 \implies P_X(h\inv(A)) =0 \implies P_Y(A) =0.
    \end{align*}
    The first implication is because $h$ is a non-singular measurable transformation. The second implication is because $P_x \ll \lambda$ as $X$ has a density with respect to $\lambda$. 

    To prove the second point, we first show that $p_X(a) > 0$ for all $a \in \R^d$ implies $P_X \gg \lambda$. To see this, observe that for any $A$, $\lambda(A) = \int_A \frac{1}{p_X(y)} p_X(y) dy \lambda =\int_A \frac{1}{p_X(y)} P_X(dy)$. With this we show that $P_Y \gg \lambda$. Let $A \subset \R^p$ be such that $P_Y(A) = 0$, then
    \begin{align*}
        P_Y(A) = 0 \implies P_X(h\inv(B)) = 0 \implies \lambda (h\inv(B)) = 0 \implies \lambda(B) =0.
    \end{align*}
    The second implication is because $P_X \gg \lambda$ and the third implication is because $h$ satisfies Lucin's condition. We prove that $p_Y > 0$ almost everywhere by contradiction. Because $P_Y$ and $\lambda$ are mutually absolutely continuous, there exists $q$ such that $\lambda = q P_Y$. Then because $P_Y = p_Y \lambda$, $\lambda = q p_Y \lambda$. Thus, $q p_Y$ must equal 1 almost everywhere with resepct to the Lebesgue measure, $p_Y$ must be non-zero almost everywhere.
\end{proof}
\subsection{Proof of \Cref{lem:admissibility}}\label{proof:lem:admissibility}
Without loss of generality we consider $T=2$. Recall that the do action alters the data generation model by deleting incoming edges into $u_1$. 
\begin{align*}
    \E[x_2\hid | do(u_1\hid \defeq u\hid)]&=\E[\Hx{f(x_{1}\hid )}{ g(u\hid)} + \noise_2]\\
    &=\int \E[\Hx{f(z)}{g(u\hid)} + \noise_2 \mid x_1 = z] p_{x_1}(z) dz\\
    &= \int \E[\Hx{f(x_1)}{g(u_1\hid)} + \noise_2 \mid u_1\hid = u\hid, x_1 = z] p_{x_1}(z) dz\\
    &= \int \E[x_2 \mid u_1\hid = u\hid, x_1 = z] p_{x_1}(z) dz.
\end{align*}
The second and third equalities use the fact that $\noise_2$ is independent of $x_1, u_1$.

\subsection{Proof of \Cref{lem:finite-discrete-full-coverage}} 
\label{proof:lem:finite-discrete-full-coverage}

Without loss of generality we will set $\bT = 2$ in this proof.

\subsubsection{Part 1: Identifiability}
\paragraph{Showing overlap}
We will first show that $(x_1, u_1)$ has full support, which automatically implies overlap.
Let $\umo \defeq (u_{-1}, x_{-1})$.
Because $$p_{u_1, x_1}(u, x) = \int p_{u_1, x_1 | \umo = z}(u, x) p_{\umo}(z)dz,$$ it suffices to show that $p_{u_1, x_1 | \umo = z}$ has full support for any $z \in \R^{p + d}$. For this reason, in this proof, we fix $\umo$---i.e., $u_{-1}$ and $x_{-1}$ will be treated like constants---and for notional simplicity, we omit explicitly conditioning on the event $\umo=z$. Let $c \defeq r(u_{-1})$, $d\defeq g(u_{-1}) + f(x_{-1})$, and $\noise_0' \defeq \noise_0 + d$. Observe that $(\noise_0', \noise_1)$ still has full support. Using this modified notation, we have
\begin{align*}
    x_0 &= \noise_0'\\
    u_0 &= \Hu{h(\noise_0')}{ c}\\
    x_1 &= \Hx{f(\noise_0')}{ g(\Hu{h(\noise_0')}{ c)}} + \noise_1\\
    u_1 &= \Hu{h(x_1)}{ r(\Hu{h(\noise_0')}{ c})}
\end{align*}
We first show that $x_1$ has full support. 
Recall $\noise_1$ has positive density over $\R^d$. 
Because addition by a constant is an invertible, differentiable function,
\Cref{lem:change-of-variables} implies that $\Hx{f(\noise_0')}{ g(\Hu{h(\noise_0')}{ c})} + \noise_1 | \noise_0'$ has positive density over $\R^d$. Since $\noise_0'$ also has positive density over $\R^d$, integration tells us that $x_1 = \Hx{f(\noise_0')}{ g(\Hu{h(\noise_0')}{ c})} + \noise_1$ has positive density over $\R^d$. 

Because $p_{x_1\hid, u_1\hid} = p_{u_1\hid|x_1\hid} p_{x_1\hid}$ and $x_1\hid$ has full support, it suffices to show that $u_1\hid | x_1\hid$ has full support over $\R^p \times \R^d$. 
It is sufficient to show that $p_{q_c(\noise_0') | x_1\hid}$ is positive everywhere. 
To see this, 
observe that $u_1\hid|x_1\hid = \Hu{h(x_1)}{ q_c(\noise_0') } | x_1\hid$. Because addition by a constant is an invertible, differentiable function, if $q_c(\noise_0') | x_1\hid$ had positive density everywhere, then  
\Cref{lem:change-of-variables} tells us that $u_1\hid | x_1\hid$ would have positive density everywhere.
One can show that the class of continuously differentiable, surjective functions with either full row-rank or full column rank Jacobian satisfy \Cref{def:lucin,def:nonsingular} \citep{3216190, 59115}.
Thus, because $q_c$ satisfies \Cref{ass:expressive-control}, the conditions of \Cref{lem:positive-q-density} hold, and thus, it suffices to show $\noise_0' | x_1\hid$ has positive density everywhere. We observe that
\begin{align*}
    p_{\noise_0' | x_1\hid}(a, b) = \frac{p_{x_1\hid | \noise_0'}(b, a) p_{ \noise_0'}(a) }{p_{x_1\hid}(b)}.
\end{align*}
Since $x_1$ has full support, the denominator is positive. Since $\noise_0'$ has full support, $p_{ \noise_0'}(a) > 0$ as well. Finally, we had already shown earlier in the proof that $x_1| \noise_0'$ (i.e., $\Hx{f(\noise_0')}{ g(\Hu{h(\noise_0')}{ c})} + \noise_1 | \noise_0'$) has positive density everywhere as well.

\paragraph{Concluding argument}
Because $p_{u_1, x_1}$ is positive everywhere, $\E[x_2 \mid u_1 = u, x_1=z]$ is well defined. Additionally, because $x_1$ has density, $\int_{z} \E[x_2 \mid u_1 = u, x_1=z] p_{x_1}(z) dz$ is well defined as well. Finally because our model is admissible as stated in \Cref{lem:admissibility}, $\E[x_2 \mid do(u_1 \defeq u)] = \int_{z} \E[x_2 \mid u_1 = u, x_1=z] p_{x_1}(z) dz$. The right hand side of this relationship is well defined and can be computed from knowledge of the distribution of $(x_1, u_1, x_2)$;
thus, $\E[x_2 \mid do(u_1 \defeq u)]$ can be computed from the distribution of observations $(x_1, u_1, x_2)$. Because this quantity identifiable, the steerability of consumption $\steer(u, u')$ is also identifiable for any $u, u' \in \R^d$.

\subsubsection{Part 2: Unidentifiability}
\label{sec:proof-thm1-unidentifiability}
Let $P_0$ be the point mass over the $0$ vector; i.e., $x_0 = u_0 =0$. Define a measurable function $\Delta: \R^p \to \R^d$ such that $\Delta \neq 0$. For any functions $f, g, h$, define, $\hat{f}(a)\defeq  f(a) + \Delta(h(a))$, $\hat g(b) \defeq g(b) - \Delta(b)$, and $\hat{h}(c) = h(c)$. For noise variables $(\noise_1, \noise_2)$, let $(\hat\noise_1, \hat\noise_2)$ an identically distributed copy.
Let $R_2 = (x_1, u_1, x_2)$ be sampled 
according to the dynamics specified by \eqref{eqn:gen_model} using the functions $f, g, h$, noise variables $(\noise_1, \noise_2)$, and with initial conditions $x_0 =u_0 =0$. Let $\hat{R}_2 = (\hat x_1, \hat u_1, \hat x_2)$ be sampled 
according to the dynamics specified by \eqref{eqn:gen_model} using the functions $\hat f, \hat g, \hat h$ in place of $f, g, h$, noise variables $(\hat\noise_1, \hat\noise_2)$ in place of $(\noise_1, \noise_2)$, and with initial conditions $\hat{x}_0 = \hat{u}_0 =0$. We see that
\begin{align*}
    x_1 &\eqd \noise_1 \eqd \hat x_1\\
    u_1 &\eqd h(\noise_1) \eqd \hat h(\noise_1)\eqd \hat u_1\\
    x_2 &\eqd f(\noise_1) + g(h(\noise_1)) + \noise_2\\
    & \eqd f(\noise_1) +\Delta(h(\noise_1)) + g(h(\noise_1)) - \Delta(h(\noise_1)) + \noise_2\\
    & \eqd \hat f(\noise_1) + \hat g(\hat h(\noise_1)) + \noise_2 \eqd \hat x_2.
\end{align*}

\subsection{Proof of \Cref{thm:identifiabile-five-tuple}}\label{proof:thm:identifiabile-five-tuple}
\subsubsection{Supporting lemmas}
We first outline a series of helpful supporting lemmas. This first lemma draws an equivalence between matrices and the probability distributions induced by these matrices, allowing us to reason about one by reasoning about the other.
\begin{lemma}\label{lem:dist_equal_cond}
Let $\{ \noise_i\}_{i=1}^n$ be a set of mutually independent random vectors in $\R^d$ with full span. Let $\{ A_i\}_{i=1}^n$ be a set of deterministic matrices in $\R^{d,p}$. Let $v \in \R^d$ be a random vector in $\R^d$ mutually independent of $\{ \noise_i\}_{i=1}^n$. $A_i = 0$ for all $i \in [n]$ and $v\eqas0$ if and only if $v + \sum_{i=1}^{n} A_i \noise_i \eqas 0$.
\end{lemma}
\begin{proof}
The left to right direction is obvious. We now prove the right to left direction by cases.
Suppose $v$ is almost surely a constant vector. Suppose that only one $j\in[n]$ such that $A_j \neq 0$, then its not possible that $A_j\noise_j \eqas -v$ by definition of full span. Suppose there exists $j,k\in[n]$ such that $A_j \neq 0$ and $A_k \neq 0$. This means that $A_j \noise_j$ is almost surely not a constant. We also know that conditioned on $\{ A_i\noise_i\}_{i\neq j}$, $A_j\noise_j$ is almost surely a constant. This implies that $P_{A_j\noise_j} \neq P_{A_j\noise_j| \{ A_i\noise_i\}_{i\neq j}}$ which contradicts the assumption of mutual independence.
Suppose  $v$ is almost surely not a constant vector. Then $P_{v} \neq P_{v| \{ A_i\noise_i\}_{i\in[n]}}$ as $v$ is almost surely a constant vector conditioned on $\{ A_i\noise_i\}_{i\in[n]}$. This contradicts mutual independence.
\end{proof}

For our next lemma and for the rest of the proof, we need to define some notation. 
Consider the following variables:
\begin{gather*}
    \Theta_x \defeq 
    \begin{bmatrix}
    A^\top  \\
    B^\top 
    \end{bmatrix} \qquad
    \Theta_u \defeq \begin{bmatrix}
    C^\top  \\
    D^\top 
    \end{bmatrix}\qquad
    \noise_x \defeq \begin{bmatrix}
    \noise_1^\top \\
    \vdots\\
    0
    \end{bmatrix}\qquad
    \noise_u \defeq 0.
\end{gather*}
Let $\hat \Theta_x, \hat \Theta_u$ be defined with respect to $\hat A, \hat B, \hat C, \hat D$. Let $\hat \noise_x \eqd \noise_x$ and $\hat \noise_u \eqd \noise_u$. Let $(A, B, C, D)$ and $\noise_x, \noise_u$ induce $\Plin_\bT$ and let $(\hat A, \hat B, \hat C, \hat D)$ and $\hat \noise_x, \hat \noise_u$ induce $\hat \Plin_\bT$.
Let $x \defeq [x_1, \ldots, x_\bT]^\top $ and $u \defeq [u_1, \ldots, u_{\bT-1}]^\top $ be observations from $\Plin_\bT$ and let $\hat x$ and $\hat u$ defined with hat variables be observations from $\hat \Plin_\bT$. Finally let $z \defeq (x, u)$ and $\hat z \defeq (\hat x, \hat u)$.
Finally, we define matrices $Q_x$, $Q_u$, $\hat Q_x$, and $\hat Q_u$ such that the following relationships hold 
\begin{align*}
    x -\noise_x \eqd Q_x \Theta_x \qquad
    u -\noise_u \eqd Q_u \Theta_u\\
    \hat{x} - \hat{\noise}_x \eqd \hat{Q}_x \hat{\Theta}_x\qquad
    \hat{u} - \hat{\noise}_u \eqd \hat{Q}_u \hat{\Theta}_u.
\end{align*}
Our next lemma translates relationships about one set of dynamics matrices into relationships about the other set of dynamics relationships.

\begin{lemma}\label{lem:backward_step}
If $x - \noise_x \eqd Q_x \hat{\Theta}_x$, then $x - \noise_x \eqd Q_x \hat{ \Theta}_x \eqd \hat{Q}_x \hat{\Theta}_x \eqd \hat{x} - \hat{\noise}_x$. Similarly,
if $u - \noise_u \eqd Q_u \hat{\Theta}_u$, then $u - \noise_u \eqd Q_u \hat{ \Theta}_u \eqd \hat{Q}_u \hat{\Theta}_u \eqd \hat{u} - \hat{\noise}_u$.
\end{lemma}
\begin{proof}
Recall that the random variables in the vector $z$ corresponds to nodes in the causal directed acyclic graph shown in \Cref{fig:autoregressive-causal-graph}.
Define $\sigma: \Z \to \Z$ such that $z_{\sigma(i)}$ is in sorted DAG order with respect to the DAG in \Cref{fig:autoregressive-causal-graph} (i.e, the parents of $z_{\sigma(i)}$ have $\sigma$ indices smaller than $\sigma(i)$ and its children have $\sigma$ indices larger than $\sigma(i)$). We proceed inductively to show that $z_{\sigma(i)} \eqd \hat z_{\sigma(i)}$.

\textit{Base case:} $z_{\sigma(1)} \eqd L(\noise, \hat{\Theta}_x, \hat{\Theta}_u)$, where $L$ is some function, linear in each of its inputs. Since $\noise \eqd \hat \noise$,  we have that $z_{\sigma(1)} \eqd L(\hat \noise, \hat{\Theta}_x, \hat{\Theta}_u) = \hat z_{\sigma(1)}$; the last equality follows from definition.

\textit{Inductive step:} suppose $z_{\sigma(j)} \eqd \hat z_{\sigma(j)}$ jointly over all $j$. We know that $z_{\sigma(j+1)} \eqd L(\{z_{\sigma(i)}\}_{i < j}, \noise, \hat{\Theta}_x, \hat{\Theta}_u)$ where $L$ is linear in $\{z_{\sigma(i)}\}_{i < j}$, linear in $\noise$, linear with respect to $\hat{\Theta}_x$, and linear in $\hat{\Theta}_u$. By the inductive hypothesis we know that $z_{\sigma(j+1)} \eqd L(\{z_{\sigma(i)}\}_{i < j}, \noise, \hat{\Theta}_x, \hat{\Theta}_u)$ which in turn is equal in distribution to $L(\{\hat z_{\sigma(i)}\}_{i < j}, \hat \noise, \hat{\Theta}_x, \hat{\Theta}_u) \eqd \hat z_{\sigma(j+1)}$, as all the inputs to the function are equal in distribution.

Because the entries of $Q$ ($\hat Q$ respectively) are comprised of entries of $z$ ($\hat z$ respectively), we have that $\hat{Q} \eqd Q$. This proves the desired result.
\end{proof}

\subsubsection{Part 1: Unidentifiability when $\bK = 1$}
Without loss of generality, let $\bT=2$. The proof of this result proceeds exactly as the proof of the unidentifiability result in \Cref{lem:finite-discrete-full-coverage} in \Cref{sec:proof-thm1-unidentifiability} except with $f, g, h, r$ defined as in \Cref{eqn:lin_dynamics} and with $\Delta: \R^p\to \R^d$ set to any linear function $\Delta(x) = Wx$ where $W \in \R^{d, p}$ is such that $W \neq 0$.

\subsubsection{Parts 2 and 3: Identifiability when $\bK \geq 2$}
Now that we have established our supporting lemmas, we can now prove our desired result. Without loss of generality, we will set $\bT = \bK + 1$.

\paragraph{Necessity and sufficiency when $x_{0} = u_{0} = \noise_t = 0$ for $t \geq 2$.}

Let $\xmat_t \in \R^{d,d}$ and $\umat_t \in \R^{p, d}$ be defined such that $x_t = \xmat_t \noise_1$ and $u_t = \umat_t \noise_1$. Further define the following random matrix:
\begin{gather*}
    Q_x \defeq \begin{bmatrix}
    x_{0}^\top  & u_{0}^\top  \\
    x_1^\top  & u_1^\top \\
    \vdots & \vdots\\
    x_{\bT-1}^\top  & u_{\bT-1}^\top 
    \end{bmatrix}.
\end{gather*}
Define hat versions of all variables accordingly. 
We have that $x - \noise_x \eqd \hat x - \hat \noise_x$ and $u - \noise_u \eqd \hat u - \hat \noise_u$. Moreover, $Q_x$ is comprised of entries of $x$ and $u$, $Q_x \eqd \hat Q_x$ (jointly). Thus,
\begin{align}\label{eqn:temp0}
    \begin{split}
    x - \noise_x&\eqd \hat x - \noise_x \eqd \hat Q_x \hat \Theta_x \eqd Q_x \hat \Theta_x\\
    u - \noise_u&\eqd \hat u - \noise_u \eqd \hat Q_u \hat \Theta_u \eqd Q_u \hat \Theta_u.
    \end{split}
\end{align}
Finally, defining the fixed matrices $\xmat \defeq [\xmat_2, \ldots, \xmat_\bT]^\top $, $\umat \defeq [\umat_2, \ldots, \umat_{\bT-1}]^\top $, and
\begin{align*}
    Q_\xmat \defeq \begin{bmatrix}
    \xmat_1^\top  & \umat_1^\top \\
    \vdots & \vdots\\
    \xmat_{\bT-1}^\top  & \umat_{\bT-1}^\top 
    \end{bmatrix},
\end{align*}
we can rewrite \eqref{eqn:temp0} as
\begin{align}\label{eqn:temp1}
\begin{split}
    \begin{bmatrix}
    \noise_1^\top \\
    \vdots\\
    \noise_1^\top 
    \end{bmatrix}
    \odot
    X
    &=
    \begin{bmatrix}
    \noise_1^\top \\
    \vdots\\
    \noise_1^\top 
    \end{bmatrix}
    \odot 
    Q_\xmat
    \hat \Theta_x.
\end{split}
\end{align}
Note, that in this reparameterization, we omit the $x_{1} - \noise_1 = Ax_{0} + B u_{0}$, as these terms are equal to 0.
Using \Cref{lem:dist_equal_cond} we know the above equality holds if and only if the following holds
\begin{align}\label{eqn:temp2}
\begin{split}
    X
    &=
    Q_\xmat
    \hat \Theta_x.
\end{split}
\end{align}
\newcommand{\mynull}{\mathop{\rm null}}
\newcommand{\myspan}{\mathop{\rm span}}
\newcommand{\col}{\mathop{\rm col}}
\Cref{lem:backward_step} tells us $B$ is identifiable if and only if the entries of $\Theta_x$ corresponding to $B$ is unique \eqref{eqn:temp2}. Indeed, if there exists two solutions $(\hat \Theta_x,\hat \Theta_u) \neq (\Theta_x, \Theta_u)$ such that $B \neq \hat B$, we can use \Cref{lem:backward_step} to show that $\hat\Plin_\bT = \Plin_\bT$; i.e., the system is not identifiable. The other direction is trivial, as $B$ being identifiable implies that $B$ is unique.

We now give equivalent conditions for when $B$ is unique. Let $\mc S \defeq \{e_j\}_{j=d+1}^{d+p}$ where $e_j$ is the $j$th standard basis vector in $\R^{d+p}$. $B$ is unique (i.e., $\hat B = B$) if an only if $\mynull(Q_X) \perp \myspan(\mc{S})$. Indeed suppose $v \in \null(Q_X)$ is such that $v$ is not orthogonal to $\myspan(\mc{S})$, then $\hat \Theta_x = \Theta_x + v1^\top $ is also a solution to \eqref{eqn:temp2}; moreover, $\hat B \neq B$ because $v$ is not orthogonal to $\myspan(\mc{S})$. 
Conversely suppose for all $v \in \null(Q_X)$, $v$ is orthogonal to $\myspan(\mc{S})$. Then, any alternative solution $\hat\Theta_x \neq \Theta_x$ must satisfy $\mc{C}(\hat\Theta_x - \Theta_x) \perp\myspan(\mc{S})$, where $\mc{C}$ denotes the column span, which implies that $\hat B = B$.

Note that if $M$ is a full rank matrix, $M Q$ has the same null space as $Q$. Further observe that by using elementary row operations, we know that there exists full rank square matrices $M_1$ and $M_2$ such that 
\begin{align*}
    Q_\xmat &= \begin{bmatrix}
    I & C^\top \\
    \xmat_2^\top  & (C \xmat_2 + D \umat_1)^\top \\
    \vdots & \vdots\\
    \xmat_{\bT-1}^\top  & (C \xmat_{\bT-1} + D \umat_{\bT-2})^\top 
    \end{bmatrix}
    = M_1\begin{bmatrix}
    I & C^\top \\
    0 & ( D \umat_1)^\top \\
    \vdots & \vdots\\
    0 & ( D \umat_{\bT-2})^\top 
    \end{bmatrix}
    = M_2 \begin{bmatrix}
    I & C^\top \\
    0 & (D C)^\top \\
    \vdots & \vdots\\
    0 & ( D^{\bT-2} C)^\top 
    \end{bmatrix}.
\end{align*}
$M_1$ and $M_2$ are products of full rank matrices corresponding to elementary row operations. $M_2$ is constructed by repeatedly applying the fact $\umat_t =C \xmat_t + D \umat_{t-1}$. Thus, $B$ is unique if and only if $\mynull(\tilde{Q}_X) \perp \myspan(\mc{S})$ where 
\begin{align*}
    \tilde{Q}_X \defeq \begin{bmatrix}
    I & C^\top \\
    0 & (D C)^\top \\
    \vdots & \vdots\\
    0 & ( D^{\bT-2} C)^\top 
    \end{bmatrix}.
\end{align*}
This is equivalent to $\myspan(\mc{S}) \subset \mc{R}(\tilde{Q}_X)$, where $\mc R$ denotes row span, which is then equivalent to $[DC, \ldots, D^{\bT-2}C]$ being full row rank (recall $\bT = \bK + 1$).
Tracing back all the if and only if statements gives the result.

\paragraph{Sufficiency even when $x_{0} \neq 0$ and $u_{0} \neq 0$.}
In this setting, the proof for Claim 1 holds up to Equation \eqref{eqn:temp1}. Equation \eqref{eqn:temp1} changes to the following
\begin{align*}
    \begin{bmatrix}
    \noise_1^\top \\
    \vdots\\
    \noise_1^\top 
    \end{bmatrix}
    \odot
    X &+ w_1(x_{0}, u_{0}, \noise_{> 1})=
    \begin{bmatrix}
    \noise_1^\top \\
    \vdots\\
    \noise_1^\top 
    \end{bmatrix}
    \odot 
    Q_\xmat
    \hat \Theta_x + w_2(x_{0}, u_{0}, \noise_{> 1}).
\end{align*}

By \Cref{lem:dist_equal_cond}, we know that these equalities hold if and only if Equation \eqref{eqn:temp2} holds, $w_1(x_{0}, u_{0}, \noise_{>1})= w_2(x_{0}, u_{0}, \noise_{>1})$ holds.
$\mynull(Q_\xmat) \perp \myspan(\mc S)$ \textit{suffices} (but is no longer necessary as there is one other relationships we are not accounting for) in showing there is a unique B in any solution of the linear system in Equation \eqref{eqn:temp2}. 
The rest of the argument in Claim 1 follows identically.

\subsection{Proof of \Cref{thm:doubleml-linear-conv-rate}}\label{proof:thm:doubleml-linear-conv-rate}
We first introduce a helpful supporting lemma.
\begin{lemma}
Suppose $n$ samples are drawn iid from $\Plin_2$.
If $X_1 X_1^\top$ is invertable, then $\what{C} = C$ and $\what{H} = A+BC + E_2 X_1^\top(X_1 X_1^\top)\inv$. If $X_1 X_1^\top$ is invertable and $DCX_1X_1^\top C^\top D^\top$ is invertable, then $\what{B} = B- E_2X_1^\top(X_1X_1^\top)\inv X_2(DCX_1)^\top(DCX_1X_1^\top C^\top D^\top)\inv$.
\end{lemma}
\begin{proof}
Substituting $U_1 = CX_1$ and $X_2 = (A+BC)X_1 + E_2$ into the closed form solutions of $\what{C}$ and $\what{H}$ respectively gives the first result.

To get the second result, we use the fact that $\what{C} = C$ and $\what{H} = A+BC+ E_2 X_1^\top(X_1 X_1^\top)\inv$ by the first result. We observe that $U_2 = CX_2 + DU_1 = CX_2 + DCX_1$ to get that $\what{B} = (X_3 - \what{H}X_2)(DCX_1)^\top(DCX_1X_1^\top C^\top D^\top)\inv$. Then we use the fact that subtracting $BCX_2$ from both sides of the relationship $X_3 - AX_2 = BU_2$ gives us that $X_3 - (A+BC)X_2 = B(U_2 - CX_2)$. Using our invertability assumptions, this gives us $\what{B} = B - E_2X_1^\top(X_1X_1^\top)\inv X_2(DCX_1)^\top(DCX_1X_1^\top C^\top D^\top)\inv$.
\end{proof}

With this, we can analyze the quantities of interest. Let $\scov_1 = \frac{1}{n}X_1X_1^\top$. Let $Q \defeq DC\scov_1C^\top D^\top $.
\begin{align*}
    \E \left[\lfro{\what{B} - B}^2 \mid \mc{G}\right]&= \frac{1}{n^2}\tr (\E[Q\inv DC X_1 X_2^\top  (X_1X_1^\top )\inv X_1 E_2^\top  E_2 X_1^\top  (X_1X_1^\top )\inv X_2 X_1^\top  C^\top  D^\top  Q\inv])\\
    &= \frac{\sigma_2^2 d}{n} \tr (\E[Q\inv DC \scov_1 (A+BC)^\top  \scov_1\inv (A+BC) \scov_1 C^\top  D^\top  Q\inv])\\
    &\leq \frac{\sigma_2^2 pd}{n} \kappa_{DC}^2 \left(\frac{\opnorm{A+BC} }{\sigma_{\min{}} (DC)}\right)^2  \E\left[\frac{\kappa_{\scov_1}^2}{\lambda_{\min{}} (\scov_1)}\right].
\end{align*}
Rearranging and using the definition of $\tau_1$ gives the result.

If $p=d$, then $DC$ is a square, invertible matrix,
\begin{align*}
    \E \left[\lfro{\what{B} - B}^2\mid \mc{G}\right] 
    &= \frac{\sigma_2^2 d}{n} \tr[(C^\top  D^\top )\inv (A+BC)^\top \E\left[\scov_1\inv\right] (A+BC) (DC)\inv ]\\
    &\leq\frac{\sigma_2^2 d^2}{n}\left(\frac{\opnorm{A+BC} }{\lambda_{\min{}} (DC)}\right)^2 \opnorm{\E\left[\scov_1\inv\right]}.
\end{align*}
Rearranging and using the definition of $\tau_2$ gives the result.

\subsection{Proof of \Cref{thm:adjustment-estimator-nonasymptotic}}
\label{proof:thm:adjustment-estimator-nonasymptotic}

\newcommand{\hY}{\hat{Y}}
\newcommand{\hZ}{\hat{Z}}
We let $Y(\setu, \setx) \defeq \E[x_2 \mid u_1 \in \setu, x_1 \in \setx]$, $Z(\setx) \defeq  Z(\setx_\idxx)$, $\hY(\setu, \setx) \defeq \what\E[x_2 \mid u_1 \in \setu, x_1 \in \setx]$, and $\hZ(\setx) \defeq \what Z(\setx_\idxx)$. The proof proceeds by bounding each of the following terms:
\begin{align*}
    \norm{\sum_{\alpha} \hY(\setu\subidxu, \setx_\idxx)\hZ(\setx_\idxx) - \E[x_2 | do(u_1 \defeq u)]} &\leq
     \norm{\sum_{\alpha} \hY(\setu\subidxu, \setx_\idxx)\hZ(\setx_\idxx)  - \sum_{\alpha}  \hY(\setu\subidxu, \setx_\idxx) Z(\setx_\idxx) }\\
    &\ + \norm{\sum_{\alpha}  \hY(\setu\subidxu, \setx_\idxx) Z(\setx_\idxx)- \sum_{\alpha}  Y(\setu\subidxu, \setx_\idxx) Z(\setx_\idxx) }\\
    &\ + \norm{\sum_{\alpha}  Y(\setu\subidxu, \setx_\idxx) Z(\setx_\idxx) - \sum_{\alpha} \E[x_2 | u_1 =u , x_1 =x] Z(\setx_\idxx)}\\
    &\ + \norm{ \sum_{\alpha} \E[x_2 | u_1 =u , x_1 =x] Z(\setx_\idxx)- \E[x_2 | do(u_1 \defeq u)]}.
\end{align*}

\subsubsection{Supporting lemmas}
We begin with a series of supporting lemmas that will aid us in bounding these terms.

\begin{lemma}\label{lem:cover-samples}
Let the conditions of \Cref{lem:finite-discrete-full-coverage} hold and let $\lambda$ denote the Lebesgue measure for $\R^{d+p}$. 
For all $A \in \familyx$ and $B\in \familyu$, the following implication is true: $\lambda(A \times B) > 0 \implies (x_1 \in A, u_1 \in B) >0$.
\end{lemma}
\begin{proof}
\begin{align*}
    P(x_1 \in A, u_1 \in B) =  \int_B \int_{A} p_{x_1, u_1}(x, u) dx du > 0
\end{align*}
We know the RHS is positive because the function being integrated is positive by \Cref{lem:finite-discrete-full-coverage}  and the set it's being integrated over has measure greater than 0.
\end{proof}

\begin{lemma}\label{lem:cover-stability}
Let $f: \R^d \to \R^p$ be a $L$-Lipschitz function. If every element of $\familyx$ has diameter at most $\varepsilon$ with respect to $\norm{\cdot}$, then for all $\setx \in \familyx$, for all $x, y \in \setx$, $\norm{f(x) - f(y)} \leq L \varepsilon$.
\end{lemma} 
\begin{proof}
Follows directly from definitions of diameter and Lipschitz Continuity.
\end{proof}

\begin{lemma}\label{lem:treatment-effect-approximation}
Consider the data generation model of \eqref{eqn:gen_model}. Let \Cref{ass:lipschitz-conditional-ev} hold. Let $x_1$ have full support. 
Then, 
\begin{align*}
    \norm{ \E[x_2 | do(u_1\defeq u)] - \sum_{\alpha} \E[x_2 | u_1 =u , x_1 =x] Z(\setx_\idxx)} \leq L \varepsilon + \norm{\int_{ \R^d \setminus \xdom} \E[x_2 \mid u_1 = u, x_1 = z] p_{x_1}(z) dz}
\end{align*}
\end{lemma}
\begin{proof}
Let $B \defeq \R^d \setminus \cup\familyx$ denote the set of points not covered by $\cup\familyx$. Then, we have the following inequalities:
\begin{align*}
    &\norm{ \E[x_2 | do(u_1\defeq u)] - \sum_{\alpha} \E[x_2 | u_1 =u , x_1 =x] Z(\setx_\idxx)}\\
    &\qquad \leq \Big\| \sum_{\idxx} \int_{\setx_\idxx} \E[x_2 \mid u_1 = u, x_1 = z] p_{x_1}(z) dz  - \sum_{\idxx}\E[x_2 \mid u_1 = u, x_1 = r] Z(\setx_\idxx)\Big\|\\
    &\qquad\qquad+ \norm{\int_B \E[x_2 \mid u_1 = u, x_1 = z] p_{x_1}(z) dz}\\
    &\qquad \leq   \sum_{\idxx} \Big\| \int_{\setx_\idxx} \E[x_2 \mid u_1 = u, x_1 = z] p_{x_1}(z) dz  - \E[x_2 \mid u_1 = u, x_1 = r] Z(\setx_\idxx)\Big\|\\
        &\qquad\qquad+ \norm{\int_B \E[x_2 \mid u_1 = u, x_1 = z] p_{x_1}(z) dz}\\
    &\qquad \leq \sum_{\idxx}\int_{\setx_\idxx} \Big\|  \E[x_2 \mid u_1 = u, x_1 = z] - \E[x_2 \mid u_1 = u, x_1 = r] \Big\|p_{x_1}(z) dz\\
        &\qquad\qquad+ \norm{\int_B \E[x_2 \mid u_1 = u, x_1 = z] p_{x_1}(z) dz}\\
    &\qquad \leq L\varepsilon
        + \norm{\int_{\R^d \setminus \xdom} \E[x_2 \mid u_1 = u, x_1 = z] p_{x_1}(z) dz}.
\end{align*}
The first and second inequality is from triangle inequality. The third comes from Jensen's inequality. The fourth inequality comes \Cref{ass:lipschitz-conditional-ev} and \Cref{lem:cover-stability}. 
\end{proof}

\Cref{lem:treatment-effect-approximation} tells us that it suffices to create an estimator that estimates $\sum_{\idxx}\E[x_2 \mid u_1 = u, x_1 = r] Z(\setx_\idxx)$---supposing that $\xdom$ is a good approximation of $\R^d$ with respect to $x_1$.

\begin{lemma}\label{lem:y-control}
Consider the data generating process from \eqref{eqn:gen_model}. Let $x_1, u_1$ have full support. Let \Cref{ass:conditional-density-control} hold, then
\begin{align*}
    \norm{\E[Y(\setu\subidxu, \setx\subidxx)] - \E[x_2|u_1 =u, x_1=x]} \leq \frac{2\eta(\varepsilon)}{1 - \eta(\varepsilon)} \E[\norm{x_2} | u_1=u , x_1 = x].
\end{align*}
\end{lemma}
\begin{proof}
Fix any $u \in \udom, x\in \xdom$. Let $Z\defeq(x_1, u_1)$, $z \defeq (x, u)$, and $A \defeq \setx\subidxx \times \setu\subidxu$. Observe that $\E[Y(\setu\subidxu, \setx\subidxx)] = \E[x_2\mid Z\in A]$. Note that these conditional expectations exist because $Z$ has full support and by construction $A$ has positive Lebesgue measure. The following holds
\begin{align*}
    \norm{\E[x_2\mid Z\in A] - \E[x_2\mid Z=z]}
    &= \norm{\int_{\R^d} x\left[ \frac{P(Z \in A | x_2 = x)}{P(Z \in A)} - \frac{p(Z=z |x_2=x)}{p(Z=z)}\right] p_{x_2}(x)  dx }\\
    &\leq \int_{\R^d} \norm{ x} \absval{ \frac{P(Z \in A | x_2 = x)}{P(Z \in A)} - \frac{p(Z=z |x_2=x)}{p(Z=z)}}p_{x_2}(x)  dx\\
    &\leq \frac{2\eta(\varepsilon)}{1 - \eta(\varepsilon)} \int_{\R^d} \norm{ x}\frac{p(Z=z |x_2=x)}{p(Z=z)}p_{x_2}(x) dx\\
    &= \frac{2\eta(\varepsilon)}{1 - \eta(\varepsilon)} \E[\norm{x_2} | Z = z].
\end{align*}
The first inequality is an application of Jensen's inequality. The second inequality is an application of \Cref{ass:conditional-density-control} and the fact that the diameter of $A$ is no more than $\varepsilon$. 
\end{proof}

\subsubsection{Applying lemmas to bound terms}
Armed with these lemmas we can proceed with bounding each of the aforementioned terms.\\~\\
\textbf{First term:}
Recall that the following holds for a $\tau^2$-subgaussian random variable $X$ 
\begin{align*}
    P(|X - \E[X]| > \delta |\E[X]|) \leq 2 \exp\left(\frac{-\delta^2 \E[X]^2}{2\tau^2}\right).
\end{align*}

For any $\idxx$, $\hZ(\setx_\idxx)$ is $\frac{1}{4n}$ subgaussian. This means we need $n = \frac{1}{2 \delta^2 Z(\setx_\idxx)^2}\log(4|\familyx|/\rho)$ samples to get $\setu\subidxu\hZ(\setx_\idxx)$ within error of $\delta Z(\setx_\idxx)$ of $Z(\setx_\idxx)$ with probability $\rho/(2|\familyx|)$. 
Using union bound, we have that with probability with at least $1-\rho/2$, 
\begin{align*}
    &\norm{\sum_{\idxx} \hY(\setu\subidxu, \setx_\idxx)\hZ(\setx_\idxx) - \sum_{\idxx}  \hY(\setu\subidxu, \setx_\idxx) Z(\setx_\idxx) }\\
    &\leq \sum_{\idxx} \norm{\hY(\setu\subidxu, \setx_\idxx)} |\hZ(\setx_\idxx)- Z(\setx_\idxx)|\\
    &\quad\leq \delta \sum_{\idxx} \norm{\hY(\setu\subidxu, \setx_\idxx)} Z(\setx_\idxx)\\
    &\quad\leq \delta \sum_{\idxx} \norm{\hY(\setu\subidxu, \setx_\idxx) -Y(\setu\subidxu, \setx_\idxx)} Z(\setx_\idxx)
    + \delta \sum_{\idxx} \norm{Y(\setu\subidxu, \setx_\idxx)} Z(\setx_\idxx)\\
    &\quad\leq \delta \gamma + \delta \sum_{\idxx} \norm{Y(\setu\subidxu, \setx_\idxx)} Z(\setx_\idxx)\\
    &\quad\leq\delta\gamma + \delta \constzero \diam +\delta  \constzero \E[\norm{g(u_1)} \mid u_1 \in \setu\subidxu, x_1 \in \setx_\idxx]\\
    &\quad\leq \delta\gamma + 2\delta \constzero \diam
\end{align*}
where the first inequality comes from triangle inequality. The second inequality comes from subgaussianity. The third inequality is from triangle inequality. The fourth inequality is from the bound of the \textbf{Second term} below. The fifth and sixth inequalities are from
triangle inequality, compactness, and from the fact $\E[\noise_t] =0$. \\~\\
\textbf{Second term:}
For any $\idxx$, $\hY(\setu\subidxu, \setx_\idxx)$ is $\frac{\sigma^2}{n_{u, x}}$ subgaussian, which means its $\frac{d\sigma^2}{n_{u, x}}$ norm-subgaussian by Lemma 1 from \cite{JinNeGeKaJo19}. Thus, the following inequality holds
\begin{align*}
    P(\norm{\hY(\setu\subidxu, \setx_\idxx) - \E[\hY(\setu\subidxu, \setx_\idxx)]} \geq t) 
    \leq 2 \exp\left( - \frac{t^2 n_{u, x}}{2 d \sigma^2 }\right).
\end{align*}
This means we need $n_{u, x} = \frac{2d \sigma^2}{\gamma^2} \log(4|\familyx|/\rho)$ samples to get $\hY(\setu\subidxu, \setx_\idxx)$ with error $\gamma$ of $\E[\hY(\setu\subidxu, \setx_\idxx)]$ with probability $\rho/ (2|\familyx|)$. Moreover, because the conditions of \Cref{lem:cover-samples} are met, we know these requirements will hold for all $n_{u, x}$ for large enough $n$. Using union bound, we have that with probability with at least $1-\rho/2$,
\begin{align*}
    &\norm{\sum_{\idxx} \hY(\setu\subidxu, \setx_\idxx) Z(\setx_\idxx) - \sum_{\idxx} Y(\setu\subidxu, \setx_\idxx) Z(\setx_\idxx) }
    \leq \sum_{\idxx}\norm{ \hY(\setu\subidxu, \setx_\idxx) - Y(\setu\subidxu, \setx_\idxx) }  Z(\setx_\idxx) \leq \gamma 
\end{align*}
The first inequality comes from Jensen's inequality. The second comes from subgaussianity.\\~\\
\textbf{Third term:}
\begin{align*}
    &\norm{\sum_{\idxx} Y(\setu\subidxu, \setx_\idxx) Z(\setx_\idxx) - \sum_{\idxx} \E[x_2 | u_1 =u , x_1 =x] Z(\setx_\idxx)}\\
    &\qquad \leq \sum_{\idxx}\norm{ Y(\setu\subidxu, \setx_\idxx) - \E[x_2 | u_1 =u , x_1 =x] )} Z(\setx_\idxx)\\
    &\qquad \leq \frac{2\eta}{1 - \eta} \sum_{\idxx} \E[\norm{x_2} | u_1=u , x_1 = x]Z(\setx_\idxx)\\
    &\qquad \leq \frac{2\eta}{1 - \eta} \left(
    2\constzero \diam 
    + \constone \right)
\end{align*}
The first inequality comes from Jensen's inequality. The second comes from \Cref{lem:y-control}. The third inequality comes from triangle inequality. \\~\\
\textbf{Fourth term:}
Recalling that $B \defeq \R^d \setminus \cup\familyx$.
\begin{align*}
    &\norm{ \sum_{\idxx} \E[x_2 | u_1 =u , x_1 =x] Z(\setx_\idxx) - \E[x_2 | do(u_1 \defeq u)]}\\
    &\qquad\leq L \varepsilon + \norm{\int_B \E[x_2 \mid u_1 = u, x_1 = z] p_{x_1}(z) dz}\\
    &\qquad\leq L\varepsilon + \constzero \E[\norm{f(x_1)} \bindic{x_1 \in B}]
    + \constzero P_{x_1}(B) \diam.
\end{align*}
The first inequality comes from \Cref{lem:treatment-effect-approximation}. The second inequality comes from $\E \noise_2 = 0$,
triangle inequality, Jensen's inequality, and the definition of $\diam$.

Union bounding over the two events and bounding the first and second terms and combining all the inequalities gives the result.

\end{document}